\documentclass{article} 
\PassOptionsToPackage{numbers, compress}{natbib}
\usepackage[accepted]{aistats2021}
\usepackage{hyperref}
\usepackage{url}
\usepackage{amsthm}
\usepackage{amssymb,dsfont,amsmath,amsthm,mathtools,natbib}
\usepackage{enumitem}
\usepackage{bbm}
\usepackage{xcolor}

\usepackage{algorithmic}
\usepackage{algorithm}
\usepackage{dcolumn}
\usepackage{color, colortbl}

\newtheorem{theorem}{Theorem}
\newtheorem{proposition}{Proposition}
\newtheorem{lemma}{Lemma}

\newtheorem{example}{Example}

\newtheorem{corollary}{Corollary}

\definecolor{Gray}{gray}{0.9}
\definecolor{LightCyan}{rgb}{0.88,1,1}
\definecolor{LightGreen}{rgb}{1,0.88,1}
\usepackage{xspace}
\newcolumntype{d}[1]{D{.}{.}{#1}}

\renewcommand{\algorithmicrequire}{\textbf{Input:}}
\renewcommand{\algorithmicensure}{\textbf{Output:}}

\def\R{\mathbb{R}}
\def\1{\mathbbm{1}}

\def\x{\mathbf{x}}
\def\X{\mathbf{X}}
\def\u{\mathbf{u}}
\def\v{\mathbf{v}}
\def\y{\mathbf{y}}
\def\r{\mathbf{r}}

\def\z{\mathbf{z}}

\def\a{\mathbf{a}}

\def\e{\mathbf{e}}

\def\s{\mathbf{s}}

\def\p{\mathbf{p}}

\def\w{\mathbf{w}}

\newcommand{\Ex}{{\rm I\kern-.3em E}}

\def\cA{\mathcal{A}}

\def\cC{\mathcal{C}}

\DeclareMathOperator{\sign}{sign}

\newcommand{\argmin}{\mathop{\mathrm{arg\,min}}}
\newcommand{\argmax}{\mathop{\mathrm{arg\,max}}}

\newcommand{\eg}{{\em e.g.,~}}

\newcommand{\rev}[1]{\textcolor{black}{#1}}

\newcommand{\res}{\rm {res}}

\usepackage[colorinlistoftodos,bordercolor=orange,backgroundcolor=orange!20,linecolor=orange,textsize=scriptsize]{todonotes}

\usepackage{cleveref}
\begin{document}

\twocolumn[

\aistatstitle{Convergent Working Set Algorithm for Lasso with  Non-Convex Sparse Regularizers}

\aistatsauthor{
	Alain Rakotomamonjy\\
	LITIS, Univ. de Rouen\\
	Criteo AI Lab, Paris\\
	\texttt{alain.rakoto@insa-rouen.fr} \\
	\And
	R\'emi Flamary  \\
	Univ C\^ote d'Azur,CNRS , OCA Lagrange  \\
	\texttt{remi.flamary@unice.fr} \\
	\AND  
	Gilles Gasso \\
	LITIS, INSA de Rouen \\
	\texttt{gilles.gasso@insa-rouen.fr} \\
	\And
	Joseph Salmon \\
	IMAG, Universit\'e de Montpellier, CNRS \\
	Montpellier, France \\
	\texttt{joseph.salmon@umontpellier.fr} \\
}
~\\
]
\vspace{1cm}

\begin{abstract}

Non-convex sparse regularizers are common tools for learning with high-dimensional
data. For accelerating convergence of a Lasso problem using those
regularizers, a working set strategy addresses the optimization problem through an iterative
algorithm by gradually incrementing the number of variables to optimize until the
identification of the solution support. We propose in this paper the first Lasso working set algorithm 
for non-convex sparse regularizers with convergence guarantees.
The algorithm, named \emph{FireWorks}, is based on a non-convex reformulation of
a recent  duality-based approach  and leverages on the geometry of the
residuals. We provide theoretical guarantees 
showing that convergence is preserved even when the inner solver is inexact,
under sufficient decay of the error across iterations.
Experimental results demonstrate strong computational gain when using our working set strategy compared to full problem solvers for both block-coordinate descent or a proximal gradient solver.
\end{abstract}

\section{Introduction}

Many real-world learning problems are of (very) high dimension. This is the case for
natural language processing problems with very large vocabulary or  recommendation problems involving million of items.
In such cases, one way of addressing the learning problem is to consider  sparsity-inducing penalties. Likewise, when the solution of a learning problem is known to be sparse, using these penalties yield to models that can leverage this prior knowledge.
The \emph{Lasso} \citep{tibshirani1996regression} and the \emph{Basis pursuit} \citep{chen2001atomic,chenbasispursuit} were the first approaches that have employed $\ell_1$-norm penalty for inducing sparsity.

The \emph{Lasso} model has enjoyed large practical successes in the machine learning and signal processing communities \citep{shevade2003simple,donoho2006compressed,lustig2008compressed,ye2012sparse}.
Nonetheless, it suffers from theoretical drawbacks (\eg biased estimates for large coefficients of the model) which can be overcome by considering non-convex sparsity-inducing penalties.
These penalties provide continuous approximations of the $\ell_0$-(pseudo)-norm which is the true measure of sparsity.
There exists a flurry of different penalties like the \emph{Smoothly Clipped Absolute Deviation} (SCAD) \citep{fan2001variable}, the \emph{Log Sum penalty} (LSP) \citep{candes2008enhancing}, the \emph{capped-$\ell_1$ penalty} \citep{zhang2010analysis}, the \emph{Minimax Concave Penalty} (MCP) \citep{zhang2010nearly}.
We refer the interested reader to \citep{Soubies_Blanc-FeraudAubert16} for a discussion on the pros and cons of such non-convex formulations.

In addition to theoretical statistical analyses, efforts have also been made for developing  computationally efficient algorithms for non-convex regularized optimization problems.
This includes coordinate descent algorithms \citep{breheny2011coordinate}, proximal gradient descent \citep{gong2013general} or Newton method \citep{wang2019fast,rakotomamonjy2015dc}.
However, all these methods share one kind of inefficiency in the sense that they spend a similar computational effort for each variable, even when these variables will end up  being irrelevant (zero weight) in the final learnt model.
In the non-convex setting, few methods have tried to lift this issue.
One approach mixes importance sampling and randomized coordinate descent \citep{flamary2015}, while another one seeks to safely screen features that are irrelevant \citep{rakotomamonjy19a}.
Working set (also known as active set) strategy aims at focusing computational effort on a subset of relevant variables, making them highly efficient for optimization problem with sparse solutions, provided that the algorithm is able to quickly identify the ``relevant" features.
In the literature, several works on working set algorithms address this selection issue mostly for convex optimization problems such as the Support Vector Machine problem \citep{vishwanathan2003simplesvm,glasmachers2006maximum} or the Lasso problem \citep{friedman2010regularization,Tibshirani_Bien_Friedman_Hastie_Simon_Tibshirani12,johnson2015blitz,pmlr-v80-massias18a}.
Working set strategies have been extended to non-convex sparse optimization problems \citep{boisbunon2014active,boisbunon2014large} but they are purely heuristic and
lack of convergence guarantees.

In this work, inspired by the Blitz algorithm proposed by
\citet{johnson2015blitz}(see also
\citep{Massias_Gramfort_Salmon17,pmlr-v80-massias18a} for its connection with
safe screening rules) we propose a theoretically supported method for selecting
a working set in non-convex regularized sparse optimization problems.
While Blitz can only be implemented for convex problems, leveraging on
primal-dual aspects of the $\ell_1$-regularized problem, we introduce a similar
algorithm that exploits the key role of the residual in a
sparse regression problem.
Our algorithm proposes a method for selecting the variables to integrate into a working set, and provides a theoretical guarantee  on objective value decrease.
Based on these results, we provide, as far as we know, the first convergence
guarantee of working set algorithm in a non-convex Lasso setting and we show
that this convergence property is preserved in a realistic inexact setting.

In summary, our contributions are the following:
	(1) we propose a novel working set algorithm for non-convex regularized regression that selects features to integrate in the model based on a so-called ``feasible" residual;
	(2) we prove that the algorithm enjoys properties such as convergence to a stationary point,
	even when the inner solver is inexact, under sufficient decay of the error along the iterations; as such, it is the first non-convex 
	working set algorithm with such a theoretical convergence proof.
	 (3) Our experimental results show that our \emph{FireWorks} algorithm  achieves
	substantial computational gain (that can reach two orders of magnitude)
	compared to the baseline approaches with proven convergence guarantees
	and on par with the heuristic working set algorithm of \cite{boisbunon2014active}.
		
\paragraph{Notation} We denote as $\X \in \R^{n\times d}$ the design matrix. We write vectors of size $d$ or size $n$ in bold \eg $\y  \in \R^n$ or $\w \in \R^d$.
We will consider several sets and they are noted in calligraphic mode. We have set of indices, mostly noted as $\cA$, with $\cA$ being a subset of indices extracted from $\{1,\dots,d\}$ and with cardinality noted $|\cA|$.
Given a set $\cA$, $\bar \cA$ denotes its complement in $\{1,\dots,d\}$.
 Set defined by (union of) function level-set will be denoted as $\cC$, with indices defining the function.
Vectors noted as $\w_\cA$ are of size $|\cA|$ and we note $\tilde \w_\cA \in \R^d$ for the vector of component $w_{j,\cA}$ for all $j \in \cA$ and $0$
elsewhere. Finally,  $\X_\cA$ represents matrix $\X$ restricted to columns indexed by $\cA$ and we will note $\res (\w) \triangleq \y - \X\w$ and  $\res (\w_\cA) \triangleq \y - \X_\cA \w_\cA = \y - \X \tilde\w_\cA$.

\section{Linear regression with non-convex regularizers}
\label{sec:framework}
We first introduce the non-convex Lasso problem we are interested in as well as
its first-order optimality conditions. We emphasize on the form of the
optimality conditions which will be key for designing our working set algorithm.
 
\subsection{The optimization problem}
\label{sub:the_optimization_problem}

We consider solving the problem of least-squares regression with a generic penalty of the form
\begin{equation}\label{eq:generalprob}
\min_{\w\in \R^d} f(\w) \triangleq \frac{1}{2} \| \y - \X\w\|_2^2 + \sum_{j=1}^d r_\lambda(|w_j|) \enspace,
\end{equation}
where $\y \in \R^n$ is a target vector, $\X=[\x_1,\dots,\x_d] \in \R^{n \times
d}$ is the design matrix with column-wise features $\x_j \in \R^n$, $\w$ is the
coefficient vector of the model and the map $r_\lambda: \R_+ \mapsto \R_+$ is
{monotonically non-decreasing,} concave and differentiable on $[0,+\infty)$ with a regularization
parameter $\lambda > 0$.
In addition, we assume that $r_{\lambda}(|\cdot|)$ is a lower semi-continuous function.
Note that most penalty functions such as SCAD, MCP or log sum (see their definitions in Table \ref{tab:ncvx_pen} in the supplementary material) satisfy such a property and that
for these penalties, $f(\cdot)$ is lower bounded.

We consider tools such as Fr\'echet subdifferentials and limiting-subdifferentials \citep{kruger2003frechet,rockafellar2009variational,mordukhovich2006frechet} well suited for non-smooth and non-convex optimization, so that a vector $\w^\star$ belongs to the set of minimizers (not necessarily global) of Problem \eqref{eq:generalprob} if  following Fermat's condition holds (see Definition 1.1 and Proposition 1.2 in \citep{kruger2003frechet} and Chapter 9 of \citep{schirotzek2007nonsmooth}):
\begin{equation}\label{eq:fermat}
\forall j,\,\,\x_j^\top (\y - \X \w^\star) \in  \partial r_{\lambda}(|w_j^\star|) \enspace,
\end{equation}
with $\partial r_\lambda(|\cdot|)$ being the Fr\'echet subdifferential of $r_{\lambda}(|\cdot|)$, assuming it exists at $\w^\star$. In particular, this is the case for the MCP, log sum and SCAD penalties presented in \Cref{tab:ncvx_pen}.
For the sake of clarity, we present next the optimality conditions for MCP and log sum.

\begin{example} For the MCP penalty (see \Cref{tab:ncvx_pen} for its definition and
	its subdifferential), it is easy to show that $\partial r_{\lambda}(|0|)= [-\lambda, \lambda]$. Hence,
	 Fermat's condition becomes with the residual 		\rev{$\res(\w^\star)$}
	\begin{align}\label{eq:optcondmcp}
	\begin{cases}
		- \x_j^\top \rev{\res(\w^\star)}  = 0, \quad & \text{if ~} |w_j^\star| > \lambda \theta \\
		- \x_j^\top \rev{\res(\w^\star)}  + \lambda \sign(w_j^\star) 
				= \tfrac{w_j^\star}{\theta}, \quad & \text{if ~}  0< |w_j^\star| \leq \lambda \theta  \\
			|\x_j^\top \rev{\res(\w^\star)}| \leq {\lambda},  \quad & \text{if ~} w_j^\star = 0 \enspace \\
	\end{cases}
	\end{align}
\end{example}
\begin{example} For the log sum penalty,
		one can explicitly compute
	$\partial r_{\lambda}(|0|) = [-\frac{\lambda}{\theta},\frac{\lambda}{\theta}]$ and leverage the smoothness of $r_{\lambda}(|w|)$
	when $|w|>0$ for computing $\partial r_{\lambda}(|w|)$. Then, the condition in Equation \eqref{eq:fermat} can be written as:
	\begin{align}\label{eq:optcond:lsp}
	\begin{cases}
		- \x_j^\top \rev{\res(\w^\star)}  + \lambda \frac{\sign(w_j^\star)}{\theta + |w_j^\star|} = 0, \quad & \text{if ~} w_j^\star \neq 0 \enspace,\\
		|\x_j^\top \rev{\res(\w^\star)}| \leq \frac{\lambda}{\theta},  \quad & \text{if ~} w_j^\star = 0 \enspace.
	\end{cases}
	\end{align}
\end{example}

As we can see, first-order optimality conditions lead to simple equations and inclusions. More interestingly, one can note that regardless of the regularizer, the structure of
optimality condition for a weight $w_j^\star = 0$ depends on
the correlation of the feature $\x_j$ with the optimal residual $\rev{\res(\w^\star)}=\y - \X\w^\star$. Hence, these conditions can be used for defining a region in which the optimal residual has to live in.

\section{Working set algorithm and analysis}
\label{sec:algorithm}

Before presenting the \emph{FireWorks} algorithm, we first introduce all concepts needed for defining and analyzing  our working set algorithm.

\subsection{Restricted problem and optimality}
\label{sub:restricted_problem}

Given a set $\cA$ of $m$ indices belonging to $\{1,\dots,d\}$, the problem defined in Equation {\eqref{eq:restrictedprob}}~is the restriction of Problem \eqref{eq:generalprob} to the columns of $\X$ indexed by $\cA$:
\begin{align}\label{eq:restrictedprob}
	 	 \min_{\w_{\cA}\in \R^{|\cA|}}
	 \frac{1}{2} \| \y - \X_\cA \w_{\cA}\|_2^2
	  	 + \sum_{j =1}^{|\cA|} r_\lambda(|w_{j, \cA}|) \enspace.
\end{align}
Naturally, a vector $\w_{\cA}^\star$ minimizing this problem has to satisfy its own optimality condition.
However, the next proposition derives a necessary condition for  optimality, that will be useful for characterizing whether $\tilde \w^\star_\cA$ is optimal for the full problem.

\begin{proposition} \label{prop:optimality}If $\w_\cA^\star$ satisfies Fermat's condition of Problem \eqref{eq:restrictedprob}, then for all $j \in \cA$, we have
\begin{equation}\label{eq:proplowerbound}
	|\x_j^\top (\y - \X_{\cA} \w_\cA^\star)| \leq  r_\lambda^\prime(0)
\end{equation}
where $r_\lambda^\prime$ is the derivative of $r_\lambda$.
\end{proposition}

 Now given  Proposition \ref{prop:optimality}, we are going to define some sets useful for characterizing
 candidate stationary points of either Equations \eqref{eq:generalprob} or \eqref{eq:restrictedprob}.
 Let us we define the function $h_j : \R^n \rightarrow \R$, for $j \in \{1,\dots,d\}$ as $h_j(\a) = |\x_j^\top \a| - r_\lambda^\prime(0)$ and the convex sets
$\cC_{j}$   as the slab   
$$\cC_{j} \triangleq \{\a \in \R^n : h_j(\a) \leq 0\}$$
and  $\cC_{j}^{=}$ as its boundary 
  $$\cC_{j}^{=} \triangleq \{\a \in \R^n : h_j(\a) = 0\}.$$
 By introducing\footnote{For $\ell_1$-type convex regularizers $\cC$ is the dual feasible set.} $\cC = \bigcap_{j=1}^d \cC_j$ and $\cC_{\cA} = \bigcap_{j \in \cA} \cC_j$ the necessary optimality condition defined in Proposition \ref{prop:optimality} can be written as $\y - \X_\cA\w_{\cA}^\star \in \cC_\cA$.
 Hence, assuming that $\w_{\cA}^\star$ is a minimizer of
its restricted Problem \eqref{eq:restrictedprob}, its extension $\tilde \w_{\cA}^\star \in \R^d$ satisfies Fermat's condition of the full problem if the following holds
\begin{equation}\label{eq:restrictedtofull}
\y - \X \tilde\w_{\cA}^\star \in \cC_{\bar \cA}\enspace,
\end{equation}
where $\bar \cA$ is the complement of $\cA$ in $\{1,\dots,d\}$. Indeed,
since $\w_{\cA}^\star$ is optimal for the restricted problem, Fermat's condition is already satisfied for all $j \in \cA$. Then, the above condition ensures that $\forall j \in \bar \cA$, we have $|\x_j^\top(\y - \X \tilde \w_{\cA}^\star)| \leq r_\lambda^\prime(0)$ since, as by definition,   $\tilde w_j^\star=0$, $\forall j \in \bar \cA$.

Equation \ref{eq:restrictedtofull} provides  an easy way to check whether a solution of a restricted problem is a potential candidate for being also a solution to the full problem. For
this purpose, 
we define the distance of a vector $\r \in \mathbb{R}^n$ to the convex set $\cC_{j}$ and  $\cC_{j}^{=}$ as
  \begin{equation}
 	\text{dist}(\r, \cC_{j}) \triangleq\min_{\z \in \R^n} \|\z - \r\|_2 \enspace, 						   \text{~s.t.~}   h_j(\z) \leq 0\enspace; \nonumber
\end{equation}
and
\begin{equation}
 						\text{dist}_{S}(\s, \cC_{j}^=) \triangleq\min_{\z \in \R^n} \|\z - \s\|_2 \enspace,
 						   \text{~s.t.~}   h_j(\z) = 0 \enspace.
 						   \nonumber
 \end{equation}
These distances  can also be used for defining the most violated
optimality condition, a key component of the methods proposed by \cite{boisbunon2014active,flamary2015}.
Indeed, given a set $\cA$, the solution $\w_{\cA}^\star$ of Equation \eqref{eq:restrictedprob} and the associated residual
\rev{$\res(\w_{\cA}^\star)$}, the index
$j^\star = \argmax_{j \in \bar \cA} \text{dist}\big(\res(\w_{\cA}^\star), \cC_{j}\big)$ is the index of the most violated optimality condition among non-active variables for the residual \rev{$\res(\w_{\cA}^\star$)}.

\subsection{Feasible Residual Working Set Algorithm for non-convex Lasso}
\label{sub:blitz_as_a_working_set_algorithm_for_non_convex_lasso}

A working set algorithm for solving Problem \eqref{eq:generalprob} consists
in sequentially solving a series of restricted problem as defined in Equation \eqref{eq:restrictedprob} with a sequence of
working sets  $\cA_0, \cA_1, \dots,\cA_k$.
The main differences among working set algorithms lie on  the way the set
is being updated. For instance, the approach of \cite{boisbunon2014active},
denoted in the experiment as MaxVC, selects the variable with the most violated
optimality conditions (as defined above) in the non-active set to be included in
the new working set, leading to the algorithm presented in the supplementary
material. \citet{flamary2015} followed a similar approach but considered a
randomized selection in which the probability of selection is related to
$\text{dist}\big(\res(\w_{\cA_k}^\star), \cC_{j}\big)$.

Our algorithm is inspired by Blitz \citep{johnson2015blitz} which is a working set algorithm dedicated to convex constrained optimization problem. But as the problem we address is a non-convex one, we manipulate different mathematical objects that need to be redefined.
The procedure is presented in Algorithm \ref{alg:activeset_ncvx}.
It starts by selecting a small subset of indices 
for instance the ten indices with largest $|\x_j^\top \y|$) as initial working set and by choosing a vector $\s_1$ such that $\s_1 \in \cC = \bigcap_{j=1}^d \cC_{j}$, for instance setting $\s_1= \mathbf{0}$. From this vector $\s_1$, we
will generate a sequence $\{\s_{k}\}$ that plays a key role in the  selection of the features to be integrated in the next restricted model.
 Then, at iteration $k$, it 
solves the restricted problem with the
 set $\cA_k$ and then by computing the residual
$\r_k = \rev{\res(\w_{\cA_k}^\star)}$ with $\w_{\cA_k}^\star$ the true solution to the restricted problem.
As noted in Equation \eqref{eq:restrictedtofull}, if $\r_k \in \cC_{\bar \cA_k}$ then the vector $\tilde \w_{\cA_k}^\star$ is
a stationary point of the full problem. If $\r_k \not \in \cC_{\bar \cA_k}$, we need to update the
working set $\cA_k$. We first prune  $\cA_k$ by removing indices 
associated to zero weights in  $\w_{\cA_k}^\star$. Then, in order to add
features to the working set, we  define
$\s_{k+1}$ as the vector on the segment $[\s_{k},\r_k]$,
nearest to $\r_k$ that belongs to $\cC$.
Then, the working set is updated by integrating predictors $j$ whose
associated slab $\cC_{j}$ frontier are nearest to $\s_{k+1}$. Hence,
the index $j$ is included in the new working set if $\text{dist}_S(\s_{k+1}, \cC_{j}^=) \leq \tau_k$, where $\tau_k$ is a strictly positive term that defines the number
of features to be added to the current working set. In practice, we have
chosen $\tau_k$ so that a fixed number $n_{\rm{added}}$ of features is added to the working set $\cA_k$ at each iteration $k$.

We provide the following intuition on why this algorithm works in practice.
At first, note that by construction $\s_{k+1}$ is a convex combination
of  two vectors one of which is the residual hence justifies its interpretation
as a pseudo-residual. However, the main difference between the $\s_k$'s and $\r_k$'s is
that the former belongs to $\cC$ and thus to
any $\cC_{\bar \cA}$ while $\r_k$ belongs to $\cC$ only for a potential  $\tilde \w_{\cA_{k}}^\star$ optimal for the full problem.
Then, when $\w_{\cA_k}^\star$ is a stationary point
for the restricted problem but not for the full problem, we have
$\r_k \in \cC_{\cA}$ but $\r_k \not \in \cC$.  Hence, $\s_{k+1}$ represents a  residual candidate for
optimality and slab's frontiers near this pseudo-residual $\s_{k+1}$ can be interpreted as the slabs associated to features that need to be integrated in the working set (allowing associated weights $w_j$'s to be potentially non-zero at the next iteration).
 This mechanism for selection is shown in Figure \ref{fig:mechanism}.
\begin{algorithm}[t]
	\caption{FireWorks: Feasible Residual Working Set Algorithm}
	\label{alg:activeset_ncvx}
	\begin{algorithmic}[1]
		\REQUIRE{ $\{\X, \y\}$, $\cA_1$ active set, $\s_1 \in \cC$, a sequence of $\tau_k$ or a mechanism for defining $\tau_k$, initial vector $ \tilde \w_{\cA_{0}}$ }
		\ENSURE{$\tilde \w_{\cA_{k}}$}
		\FOR{$k=1,2,\dots $}
		\STATE $\w_{\cA_{k}} =  \argmin_{\w}  \frac{1}{2}\|\y -  \X_{\cA_k} \w \|_2^2 +  \sum_{j \in \cA_k} r_\lambda(|w_j|) $ \hfill \emph{//warm-start solver with $\w_{\cA_{k-1}}$}
		\STATE $\r_k = \y - \X_{\cA_k} \w_{\cA_{k}}$ \hfill\emph{//get residual}
		\STATE $\alpha_k = \max \{\alpha \in [0,1] :  \alpha \r_k
		+ (1- \alpha) \s_{k} \in \cC\}$
		\STATE $\s_{k+1} = \alpha_k \r_k
		+ (1- \alpha_k) \s_{k}$ \hfill\emph{//define the most "feasible" residual}
		\STATE $\cA_k = \cA_k / \{j \in \cA_j : w_{j,\cA_{k}}=0 \}$ \hfill\emph{//prune the set from inactive features}
		\STATE compute $\tau_k$  \hfill\emph{//\emph{e.g.,} sort $ \text{dist}_S(\s_{k+1}, \cC_{j}^=$) so as to keep constant number of features to add}
		\STATE $\cA_{k+1} = \{j : \text{dist}_S(\s_{k+1}, \cC_{j}^=)\} \leq \tau_k\} \cup \cA_k$ \hfill\emph{//update working set}
		\ENDFOR
		\STATE Build $\tilde \w_{\cA_{k}}$
	\end{algorithmic}
\end{algorithm}
\begin{figure*}\centering
	\includegraphics[width=.90\linewidth]{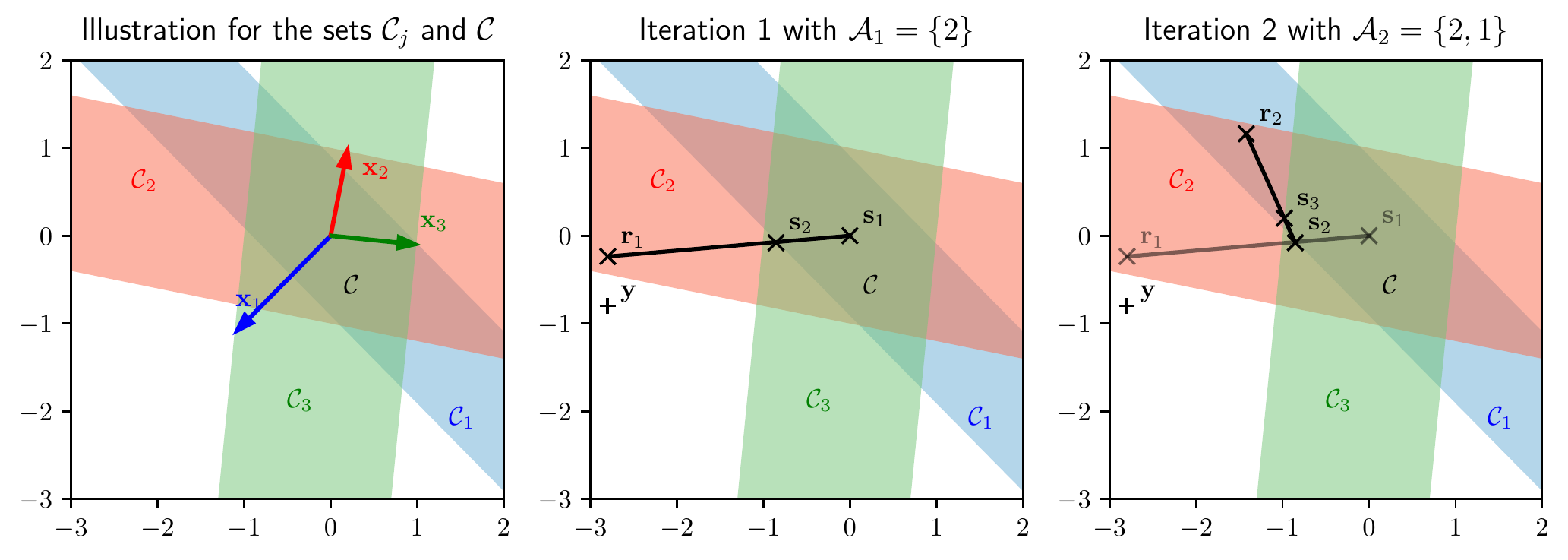}
 
\caption{Illustrating the feature selection.
(left)  Given three variables, we plot their associate slabs  $\{\cC_{j}\}_{j=1}^3$.  $\cC$ is the intersection of the
$3$ slabs.	We assume
that the initial working set is $\{2\}$.
	 (middle) After the first iteration,  the residual $\r_1$ satisfies the condition $h_2(\a) \leq 0$ and thus lies in region $\cC_2$. Then, the segment $[\s_1,\r_1]$ gives us the most feasible point $s_2 \in \cC$. If $\tau_1$ is chosen so as to select only one feature, it is then $j=1$. The new working set
	 is $\{2,1\}$. (right)
	  After optimizing over this working set, the residual $\r_2$ lies in the $\cC_1 \cap \cC_2$ region.
\label{fig:mechanism}
}
\end{figure*}

\paragraph{Relation with maximum violated optimality condition algorithm \cite{boisbunon2014active}.} The mechanism we have proposed for updating the working set is based on the current residual $\r_k$ and a feasible residual $\s_k$. By changing how $\s_{k+1}$ is defined, we can retrieve the algorithm
proposed by \citet{boisbunon2014active}.
Indeed, if we set at~Line~5 of Algorithm \ref{alg:activeset_ncvx}, $\forall k, \s_k=0$ and $\s_{k+1} = \alpha_k \r_k$, with $\alpha_k \in[0,1]$ then $\s_{k+1}$ is a rescaling of the current residual and the scale is chosen so that $\s_{k+1} \in \cC$.
Using a simple inequality argument, it is straightforward to show that
$\alpha_k = \min(\min_{j \in \bar \cA_k} \frac{\lambda}{|\x_j^\top r_k|},1)$
and the minimum in $j$ occurs for the largest value of $|\x_j^\top \r_k|$. From the theoretical side, we want to emphasize that \citet{boisbunon2014active} do not provide convergence proof of this
algorithm. Nonetheless, we conjecture that the polynomial convergence of this algorithm is guaranteed for exact inner solver and when working set is never pruned (removing from the set $\cA_k$ variables which weights are $0$ is not allowed).

\subsection{Some properties of the algorithm}
\label{sub:some_properties_of_the_algorithm}

In this subsection, we analyze some properties of the proposed algorithm. At first, we introduce an alternative optimality condition (whose proof is in the supplemental), based on $\alpha_k$ for the
full problem. Based on this property and some intermediate results, we will show that the iterates
$\tilde \w_{\cA_{k}}^\star$ converge towards a stationary point of the full problem. 

\begin{proposition} \label{prop:optalpha} Given a working set $\cA_{k}$ and $\w_{\cA_{k}}^\star$ solving the related restricted problem, $\tilde \w_{\cA_{k}}^\star$ is also optimal for the full problem if and only if $\alpha=1$ in Algorithm \ref{alg:activeset_ncvx}, step 4 (which also means $\s_{k+1} = \r_k$).
\end{proposition}

Now, we are going to characterize the decrease in objective value obtained between two updates of working sets, assuming that in the update, there is a least one feature that does not satisfy its optimality condition.

\begin{proposition} \label{prop:w}Assume that $\|\X\|_2 > 0$ and $\w_{\cA_{k}}^\star$ and $\w_{\cA_{k+1}}^\star$ are respectively the solutions of the restricted problem with the working set $\cA_k$ and $\cA_{k+1}$, with $\cA_{k+1} =  \{j_1,\cdots, j_{n_{added}}\} \cup \cA_k$, such that there exists at least one $j_i$ with ${\rm dist}(\r_k, \cC_{j_i}) > 0$.
	As we note $\r_k \triangleq   \rev{\res(\w_{\cA_k}^\star)}$,  	the following inequality holds for all $j_i$ such that ${\rm dist}(\r_k, \cC_{j_i}) > 0$
	\begin{align*}
	\| \tilde \w_{\cA_{k+1}}^\star - \tilde\w_{\cA_{k}}^\star\|_2 \geq \frac{1}{\|\X\|_2} {\rm dist}(\r_k, \cC_{j_i}) \enspace.
	\end{align*}
	
\end{proposition}
\begin{proof}
	We have the following inequalities
	\begin{equation}
	\|\r_{k+1} - \r_k\|_2
	= \|\X (\tilde \w_{\cA_{k+1}}^\star - \tilde \w_{\cA_{k}}^\star) \|_2 
	\leq \|\X\|_2 \|\tilde\w_{\cA_{k+1}}^\star - \tilde\w_{\cA_{k}}^\star\|_2 \label{eq:dist1} \enspace.
	\end{equation}
	
	Now recall that $\r_k \not\in  \cC_{\cA_{k+1}}$ since $\exists j_i :$ ${\rm dist}(\r_k, \cC_{j_i}) > 0$, while $\r_{k+1} \in  \cC_{\cA_{k+1}}$ as $\w_{\cA_{k+1}}^\star$
	has been optimized over $\cA_{k+1}$. As such, for all
	$j_i :$ ${\rm dist}(\r_k, \cC_{j_i}) > 0$, we also have
	$h_j(\r_{k+1}) \leq 0$.
	Now by definition of
	$\text{dist}(\r_k, \cC_{j_i})$ either $\r_{k+1}$ is the minimizer of the distance optimization problem, hence $\text{dist}(\r_k, \cC_{j_i}) = \|\r_{k+1} - \r_k\|_2$ or  $\text{dist}(\r_k, \cC_{j}) \leq \|\r_{k+1} - \r_k\|_2$.
	Plugging this latter inequality in \ref{eq:dist1} concludes the
	proof.
\end{proof}
Given the right hand side of the equation in Proposition \ref{prop:w},
we now show
 that the distance of the residual $\r_k$ at step $k$ to a set  $\cC_{j}$, defined by a feature $j$ that is not yet in the active set,  is lower bounded by a term depending on the parameter $\tau_{k-1}$ which governs the number of features that has been added to the active set at step $k-1$.
\begin{lemma}\label{prop:dist} At step $k \geq 2$, consider a set $\cC_{j}$ such that
	$h_j(\r_k) >0$ and $h_j(\s_k) <0$, then
	\begin{align}
	    \text{dist}(\r_k, \cC_{j}) \geq \frac{1 - \alpha_k}{\alpha_k} \tau_{k-1} \enspace.
	\end{align}
\end{lemma}
The proof of this lemma is available in the supplementary material.
From the above Proposition \ref{prop:w} and Lemma \ref{prop:dist}, we can ensure that the sequence
$\{\tilde \w_{\cA_{k}} \}$ produced by Algorithm \ref{alg:activeset_ncvx} converges
towards a stationary point under mild conditions on the inner solver.

\begin{theorem}\label{prop:convergence} Suppose that for each step $k$, the algorithm solving the inner problem ensures	a decrease in the objective value in the form
\begin{align*}
	f(\tilde \w_{\cA_{k+1}}^\star) - f(\tilde \w_{\cA_k}^\star) \leq - \gamma_k \|\tilde \w_{\cA_{k+1}}^\star - \tilde \w_{\cA_k}^\star\|_2^2 \enspace.
\end{align*}
with $\forall k, \,\gamma_k \geq \underline{\gamma} > 0$. For the inner solver, we also impose
that when solving the problem with set $\cA_{k+1}$, the inner solver is warm-started with $\w_{\cA_{k}}^\star$. Assume also that
$\|\X\|_2 > 0$,  $\tau_k \geq \underline{\tau} >0$  and $h_j$ satisfies assumption in Lemma \ref{prop:dist}, then the sequence of $\alpha_k$ produced by Algorithm \ref{alg:activeset_ncvx} converges towards $1$ and $\forall j,\,\,\lim_{k \rightarrow \infty} |\x_j^\top \r_{k}| \leq r^\prime_\lambda(0)$.
\end{theorem}

The above theorem ensures convergence to a stationary point under some conditions on the inner solver and on  the $\gamma_k$'s which needs to be lower bounded by  $\underline{\gamma} > 0$ .
Several algorithms may satisfy this assumption.
For instance, any first-order iterative algorithm which selects its step size as $t_k$ based on line search criterion of the form
	$\forall k,\,\,f(\w_{k+1})
	\leq
	f(\w_k) - \frac{\sigma}{2} t_k \|\w_{k+1} - \w_{k}\|_2^2 \enspace,
$
where $\sigma$ is a constant in the interval $(0,1)$,
provides such a guarantee. This is the case of the generalized proximal
algorithm of \citet{gong2013general}[Section 2.3.2] or proximal Newton approaches \citep{rakotomamonjy2015dc}, assuming that $f$ is differentiable with
gradient Lipschitz and $r_\lambda(\cdot)$ admits a proximal operator.
Since non-convex block coordinate descent algorithms \citep{breheny2011coordinate}
can also be interpreted as proximal algorithm, they also satisfy this sufficient  decrease condition under the same assumptions than proximal approaches.

Another important condition for convergence is based on the parameter $\tau_k$. We note the lower bound
$\underline{\tau}$ can be set to any arbitrary small positive value. At a non-optimal for the full problem $\tilde \w_{\cA_{k}}$, and as small as this lower bound is, the set $\{j : {\rm{dist}}_S(\s_{k+1}, \cC_{j}^=) \leq \underline{\tau}\}$ always contains at least the index $j$
that makes $\alpha_k$ maximal and corresponds to the $j$ such that ${\rm{dist}}_S(\s_{k+1}, \cC_{j}^=)=0$  (see Line 4 of the algorithm). This would correspond to updating the working set by one element at each iteration.

\rev{The above theorem states about the convergence of the working set strategy. We want to emphasize here that the convergence rate of the whole algorithm \ref{alg:activeset_ncvx}
(working set + inner solver) depends on the convergence rate of the inner solver. For instance, if we consider as an inner solver the proximal algorithm
of \citet{gong2013general}, then the convergence rate for each inner problem
is of the form $C \cdot\frac{\|f(\w_0) - f(\w^\star)\|}{T}$ where $C$ is a constant depending on the inner problem, $T$ the total number of iterations for that solver, 
and $\w_0$ the initial point when solving that problem. Since Algorithm 
\ref{alg:activeset_ncvx} runs this inner solver several times, the convergence
rate is still in $\mathcal{O}(n)$ but with a different constant. The gain 
in computation time achieved by using a working set strategy comes from the fact that each inner solver involves far fewer variables than the full problem dimensionality $d$ and thus gradients are cheaper to compute.}

 \paragraph{Inexact inner solver} One key point when considering a meta-solver like Blitz \citep{johnson2015blitz} or a working set algorithm  is that for some approaches,  theoretical properties hold only when the solution of the inner solver is exact. This is for instance the case for the
 SimpleSVM algorithm of \citet{vishwanathan2003simplesvm} or the active set
 algorithm proposed by \citet{boisbunon2014active}. The convergence of these approaches are based
 on non-cyclicity of the working set selection (prohibiting pruning) and thus on the ability of solving
 exactly the inner problem.
 For the approach we propose, we show next that the distance between two consecutive inexact solutions of the inner problem is still lower bounded.

 \begin{proposition} Let $\w_{\cA_{k}}^\star$ and $\w_{\cA_{k+1}}^\star$ the approximate solutions of the inner problem  with respectively the working sets $\cA_{k}$ and $\cA_{k+1}$, as defined in Proposition \ref{prop:w}. Assume that
 $\w_{\cA_{k+1}}^\star$ has been obtained through a tolerance of $\xi_{k+1} \leq \tau_k$ of its Fermat's condition (\eg for the log sum penalty, Equation \eqref{eq:optcond:lsp} are satisfied up to $\xi_{k+1}$), then the following inequality holds :
 $$
	 \|\tilde \w_{\cA_{k+1}}^\star - \tilde \w_{\cA_k}^\star\|_2^2 \geq \frac{1}{\|\X\|_2} \big({\rm dist}(\r_k, \cC_j) - \xi_{k+1} \big).
 $$
 \end{proposition}
\begin{proof} First note that if $\w_{\cA_{k+1}}^\star$ is such that $\r_{k+1} \in \cC_{\cA_{k+1}}$
then we are in the same condition than in Proposition \ref{prop:w} and the same proof applies. Let us assume then
that $\r_{k+1} \not \in \cC_{\cA_{k+1}}$ and $\text{dist}(\r_{k+1}, \cC_{j}) \leq \xi_{k+1}$.
Define as $\u$ the point in $\cC_{j}$ that defines the distance of $\r_k$ to $\cC_j$ and
as $\p$ the point that minimizes the distance between $\r_{k+1}$ and the segment $[\u,\r_k]$. Then, owing to simple geometrical arguments and orthogonality we have : $\|\r_{k+1} - \r_{k}\|^2 = \|\r_{k+1} - \p \|^2 + \|\p - \r_{k}\|^2$ and thus $\|\r_{k+1} - \r_{k}\| \geq \|\r_{k} - \p\|$. Now, because  $\p$ belongs to the segment defined by
 $\u$ and $\r_{k}$,  we have
$$
\|\r_{k+1} - \r_{k}\| \geq  \|\r_{k} - \u\| -  \|\u - \p\| \geq
\text{dist}(\r_{k}, \cC_{j}) - \xi_{k+1}
$$
where the last inequality comes from the fact that $\|\u - \p\| =
\text{dist}(\r_{k+1}, \cC_{j}) \leq \xi_{k+1}$. Plugging this inequality into
Equation \eqref{eq:dist1} completes the proof.
\end{proof}
Note that the above lower bound is meaningful only if the tolerance $\xi_{k+1}$ is smaller than the distance of the residual to the set $\cC_{j}$. This is a reasonable assumption to be made since we expect $\r_k$ to violate $\cC_{j}$.  Now, we can derive condition of convergence towards a stationary point of the full problem.
\begin{corollary} Under the assumption of Theorem \ref{prop:convergence} and
	assuming that the sequence of tolerances $\xi_{k}$ is such that $\sum_k \xi_k < \infty$, then Algorithm \ref{alg:activeset_ncvx} produces a sequence of iterates that converges towards a stationary point.
\end{corollary}
The proof follows the same steps as for Theorem \ref{prop:convergence}, with the addition that sequence $\{\xi_k\}$ is convergent and thus has been omitted. Note that the assumption of convergent sum of errors is a common assumption, notably in the proximal algorithm literature \citep{combettes2005signal,villa2013accelerated} and it helps guaranteeing convergence towards exact stationary point instead of an approximate convergence.

\begin{table*}[t]
	\caption{Running time in seconds of different algorithms on different problems. In the first column, we reported data, the tolerance on the stopping criterion
		and the constant $K$ such that $\lambda = K \max_j |\x_j^\top \y|$ (the larger the $K$, the sparser $\w^\star$ is).
		The small \emph{Toy} dataset has  $n=100$, $d=1000$ and $p=30$; the large one has $n=1000$, $d=5000$, $p=500$. For each inner solver, we
		bold the most efficient algorithm.
		The symbol "$-$" denotes that the algorithm did not finish one iteration in $24$ hours and the $0.0$ as a standard deviation means that only one iteration were terminated after $48$ hours.. The number in parenthesis is the number of non-zero weights in $\w_\cA^\star$. All experiments have been run on one single core of an Intel Xeon CPU E5-2680 clocked at 2,4Ghz.
		\label{table:toy} }
	~\\
	\resizebox{\linewidth}{!}{
		\begin{tabular}{l|rrrr|rrrr}
			\hline
			Data and Setting & MM prox & GIST    & MaxVC Gist  & FireWorks  Gist   & MM BCD &  BCD & MaxVC BCD & FireWorks BCD	\\\hline \hline
			Toy small - 1.00e-03 - 0.07 & 1.4$\pm$0.4 (34) & 0.8$\pm$0.2 (34) & 0.3$\pm$0.2 (34) & \textbf{0.2$\pm$0.1} (34) & 3.4$\pm$0.9 (34) & 14.2$\pm$4.9 (34) & 1.9$\pm$0.8 (34) & \textbf{1.5$\pm$0.9} (34)\\
			Toy small  - 1.00e-05 - 0.07  & 1.5$\pm$0.4 (34) & 1.4$\pm$0.6 (34) & 0.7$\pm$0.8 (34) & \textbf{0.4$\pm$0.1} (34) & 3.3$\pm$0.8 (34) & 22.9$\pm$11.0 (34) & 8.3$\pm$9.7 (34) & \textbf{2.7$\pm$1.2} (34)\\
			Toy small - 1.00e-03 - 0.01 & 11.2$\pm$1.2 (71) & 6.3$\pm$2.2 (71) & 1.6$\pm$0.6 (71) & \textbf{1.3$\pm$0.6} (71) & 83.7$\pm$18.6 (71) & 73.7$\pm$21.7 (71) & 15.6$\pm$4.5 (71) & \textbf{8.2$\pm$2.0} (71)\\
			Toy small  - 1.00e-05 - 0.01 & 17.6$\pm$6.0 (66) & 14.1$\pm$9.8 (66) & 7.1$\pm$5.3 (66) & \textbf{4.6$\pm$2.8} (66) & 88.2$\pm$23.3 (66) & 154.6$\pm$93.6 (66) & 67.0$\pm$44.5 (66) & \textbf{40.8$\pm$24.1} (66)\\
			Toy large - 1.00e-03 - 0.07 & 41.1$\pm$15.3 (365) & 26.2$\pm$13.0 (365) & \textbf{5.8$\pm$1.3} (365) & {8.2$\pm$3.3} (365) & 1040.8$\pm$0.0 (365) & 355.9$\pm$83.8 (365) & 82.7$\pm$19.3 (365) & \textbf{73.5$\pm$9.7} (365)\\
			Toy large - 1.00e-05 - 0.07 & - & 50.5$\pm$7.6 (371) & 36.8$\pm$13.3 (371) & \textbf{31.7$\pm$7.4} (371) & 1356.7$\pm$178 (371) & 1030.5$\pm$471.7 (371) & 561.7$\pm$208.8 (371) & \textbf{465.6$\pm$111.4} (371)\\
			Toy large - 1.00e-03 - 0.01 & 589.5$\pm$185.4 (758) & 91.6$\pm$22.9 (758) & 65.4$\pm$14.5 (758) & \textbf{34.9$\pm$4.1} (758) & 52848.8$\pm$0.0 (758) & 1192.1$\pm$340.1 (758) & 777.5$\pm$181.5 (758) & \textbf{337.0$\pm$46.3} (758)\\
			Toy large - 1.00e-05 - 0.01 & - & \textbf{583.8$\pm$140.7} (759) & 1020.6$\pm$250.6 (759) & 609.4$\pm$177.6 (759) & 60897$\pm$5990 (759) & 7847$\pm$2774 (759) & 12720$\pm$2520 (759) & \textbf{6699$\pm$1686} (759)\\\hline

																										\end{tabular}
	}
	\\~\\~\\
	\resizebox{\linewidth}{!}{
		\begin{tabular}{l|rrrr|rrrrr}
			\hline
			Data and Setting & MM prox & GIST    & MaxVC Gist  & FireWorks  Gist   & MM BCD &  BCD & MaxVC BCD & FireWorks BCD	\\\hline \hline

			Leukemia - 1.00e-03 - 0.07 & 6.3$\pm$2.0 (7) & 17.9$\pm$0.4 (7) & \textbf{0.2$\pm$0.0} (7) & 0.4$\pm$0.0 (7) & 3.8$\pm$0.7 (7) & 144.4$\pm$1.1 (7) & \textbf{0.8$\pm$0.0} (7) & \textbf{0.8$\pm$0.0} (7)\\
			Leukemia - 1.00e-05 - 0.07 & 8.0$\pm$2.7 (9) & 26.1$\pm$0.6 (9) & \textbf{0.3$\pm$0.0} (9) & 0.5$\pm$0.0 (9) & 4.6$\pm$1.1 (9) & 218.8$\pm$1.1 (9) & 1.2$\pm$0.0 (9) & \textbf{1.1$\pm$0.0} (9)\\
			Leukemia - 1.00e-03 - 0.01 & 31.4$\pm$6.2 (41) & 186.1$\pm$1.7 (41) & \textbf{5.4$\pm$0.0} (41) & 5.5$\pm$0.0 (41) & 53.6$\pm$9.6 (41) & 1168.3$\pm$0.2 (41) & 19.9$\pm$0.0 (41) & \textbf{17.4$\pm$0.0} (41)\\
			Leukemia - 1.00e-05 - 0.07 & 71.4$\pm$7.5 (46) & 525.2$\pm$8.5 (46) & 20.3$\pm$0.0 (46) & \textbf{14.6$\pm$0.0} (46) & 65.5$\pm$4.9 (46) & 1412.8$\pm$0.3 (46) & 71.5$\pm$0.0 (46) & \textbf{42.7$\pm$0.0} (46)\\\hline

																		Newsgroup-3 - 1.00e-02 - 0.01 & 955.8$\pm$389.1 & 6041.1$\pm$7.2 & \textbf{6.5$\pm$0.0} & 8.3$\pm$0.0 & 7926.6$\pm$3183.6 & 3792.4$\pm$6.2 & \textbf{4.9$\pm$0.0} & 5.6$\pm$0.0\\
			Newsgroup-3 - 1.00e-03 - 0.01 & 1200.6$\pm$402.7 & 5790.6$\pm$8.0 & 49.8$\pm$0.1 & \textbf{36.6$\pm$0.0} & 12078.0$\pm$3879.1 & 24070.5$\pm$18 & 53.2$\pm$0.1 & \textbf{36.8$\pm$0.0}\\
			Newsgroup-3 - 1.00e-04 - 0.01  & 1237.9$\pm$415.5 & 5734.0$\pm$3.9 & 1439.3$\pm$2.4 & \textbf{326.1$\pm$0.2} & 12130.8$\pm$3849.7 & 37639.8$\pm$19 & 279.2$\pm$0.2 & \textbf{167.7$\pm$0.1}\\
															Newsgroup-5 - 1.00e-02 - 0.01 & - & 26711.1$\pm$44 & 1001.2$\pm$2.7 & \textbf{343.6$\pm$0.9} & -  & 77378.7$\pm$74 & 421.7$\pm$0.8 & \textbf{172.5$\pm$0.1}\\
			Newsgroup-5 - 1.00e-03 -  0.01& -  & 26685.6$\pm$14 & 2163.6$\pm$4.4 & \textbf{876.9$\pm$0.6} & - & 91603.9$\pm$0.0 & 728.9$\pm$2.9 & \textbf{312.3$\pm$0.6}\\
			Newsgroup-5 - 1.00e-04 - 0.01& - & 26752.5$\pm$15 & 4285.2$\pm$6.1 &   \textbf{1632.5$\pm$3.2} &  - & 117749.0$\pm$0.0 & 1093.7$\pm$3.7 & \textbf{554.2$\pm$1.0}\\
						Criteo - 1.00e-02 - 0.005 & - &  -  & - & - & - & - & 41095.3$\pm$2218 & \textbf{31052.7$\pm$1202}\\
			Criteo - 1.00e-03 - 0.005 & -  & - & - & - & - & - & 49006.7$\pm$1431 & \textbf{37534.6$\pm$1576}\\
			Criteo - 1.00e-04 - 0.005 & -  & - & - & - & - & - & 59303.8$\pm$1308 & \textbf{42773.9$\pm$1022}\\\hline
		\end{tabular}
	}
\end{table*}

\section{Numerical Experiments}
\label{sec:expe}

\paragraph{Set-up}~We now present some numerical studies showing the computational gain achieved by our approach. Our main baselines are algorithms that also feature convergence guarantees. As such,
we have considered, for solving the full problem a proximal algorithm \cite{gong2013general} and a coordinate descent approach \cite{breheny2011coordinate}; they are respectively denoted
as GIST and BCD.  We have also used those algorithms as inner solvers  into our working set algorithm, denoted as FireWorks (for FeasIble REsidual WORKing Set). All methods have been implemented in Python/Numpy \cite{numpy} and the code  will be published under MIT License.
As another baseline with theoretical convergence guarantees, we have considered a solver based on majorization-minimization (MM) approach,  which consists
in iteratively minimizing a majorization of the non-convex
objective function as in \cite{hunter2004tutorial,gasso2009recovering,rakotomamonjy19a}. Each iteration results in a weighted convex Lasso problem that we solve, after warm-starting with previous iteration result,
with a Blitz-based proximal Lasso or BCD Lasso (up to precision of $10^{-5}$ for
its optimality conditions). Our last baseline is the maximum-violating
optimality condition working set algorithm (MaxVC)  described in Algorithm
\ref{alg:maxvc_ncvx} in supplementary and that is known to be very
efficient, but does not come with a
convergence proof (though we conjecture it can be proved when
no pruning occurs).

For  all approaches, we leverage the closed-form proximal
operator available for several (non-convex) regularizers. For our experiments, we have used the log-sum penalty which has an hyperparameter
$\theta$ that has been set to $1$. For all algorithms, the stopping criterion
is based on the tolerance (either $10^{-3}$ or $10^{-5}$ ) over  Fermat's optimality condition given in Equation~ \ref{eq:fermat}
The used performance measure for comparing all algorithms is the CPU running time. For
all problems, we have set $\tau_k$ adaptively (by sorting as described in
Algorithm \ref{alg:activeset_ncvx} line 7) so as to add the same fixed number
$n_{\rm{added}}$ of features
into the working set of our FireWorks algorithm and for MaxVC. Results are averaged over $5$ different  runs.
\paragraph{Toy problem}
Here, the regression  matrix $\X \in \R^{n \times d}$ is drawn uniformly from a standard Gaussian distribution (zero-mean unit variance).
For given $n,d$ and a number {$p$} of active variables, the true coefficient vector $\w^{\rm true}$ is obtained as follows.
The {$p$} non-zero positions are chosen randomly, and their values are drawn from a zero-mean unit variance Gaussian distribution, to which we added $\pm 0.1$ according to $\sign(w_j^{\rm true})$.
Finally, the target vector is obtained as $\y = \X\w^{\rm true} + \e$ where $\e$ is a zero-mean Gaussian noise with standard deviation $\sigma=0.01$. For these problems, we have arbitrarily set $n_{\rm{added}}=30$ and extra experiments in the appendix illustrates the impact of this choice.
Table \ref{table:toy} presents the running time for
different algorithms to reach convergence under various settings. We note
that our FireWorks algorithm is faster than the genuine inner solver and (at least on par) with the MaxVC approach especially in setting where $\lambda$ is properly tuned with respect to the number of variables, \emph{ie} when the solution is not too sparse. Note that the MM+Blitz approaches is performing worse than all other methods in almost all settings. We explain this gain by the working set framework and the ability to prune
the working set, which size is therefore not monotonically increasing.

\begin{figure*}[t]\centering
	\includegraphics[width=3.9cm]{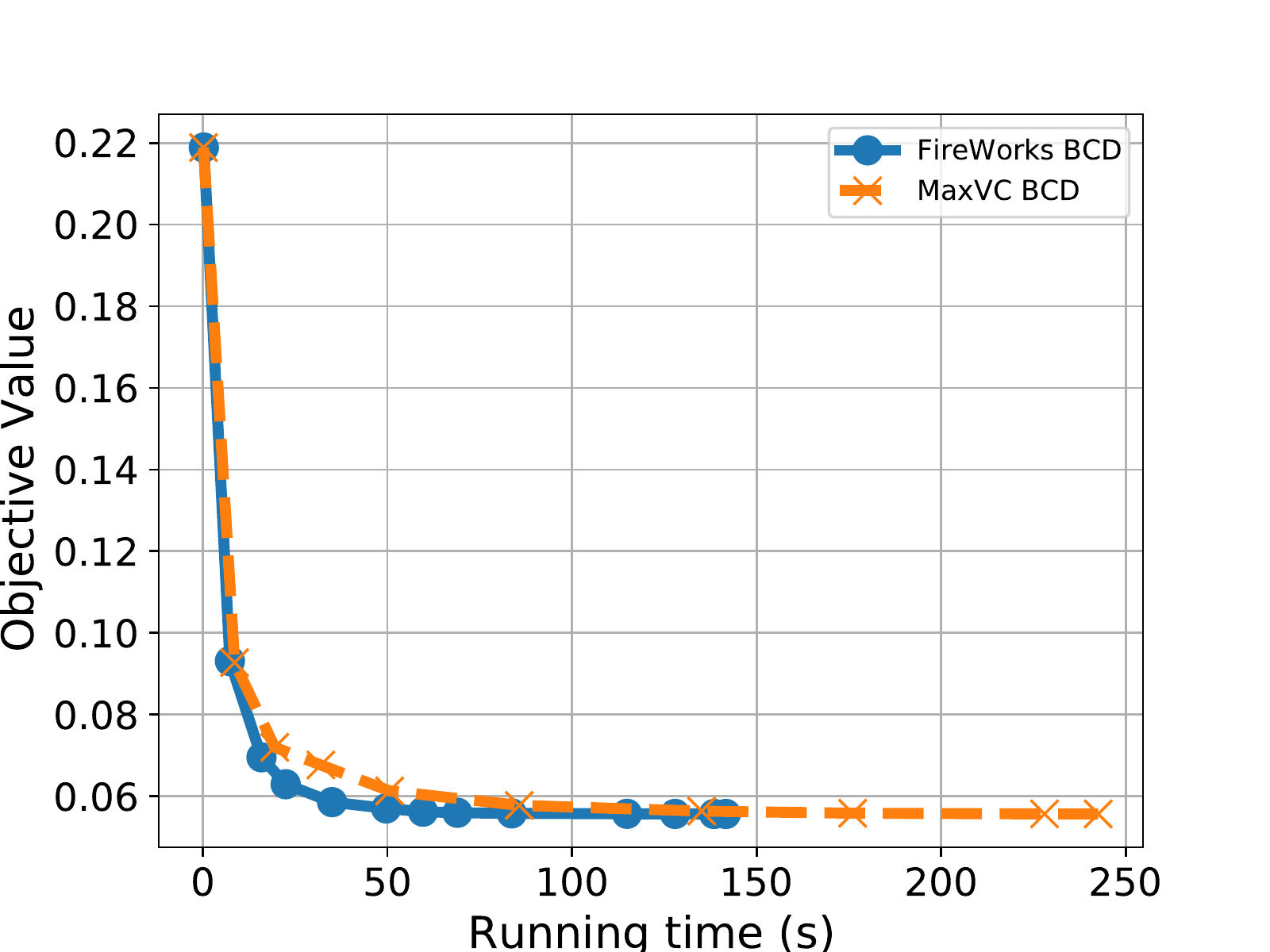}
	\includegraphics[width=3.9cm]{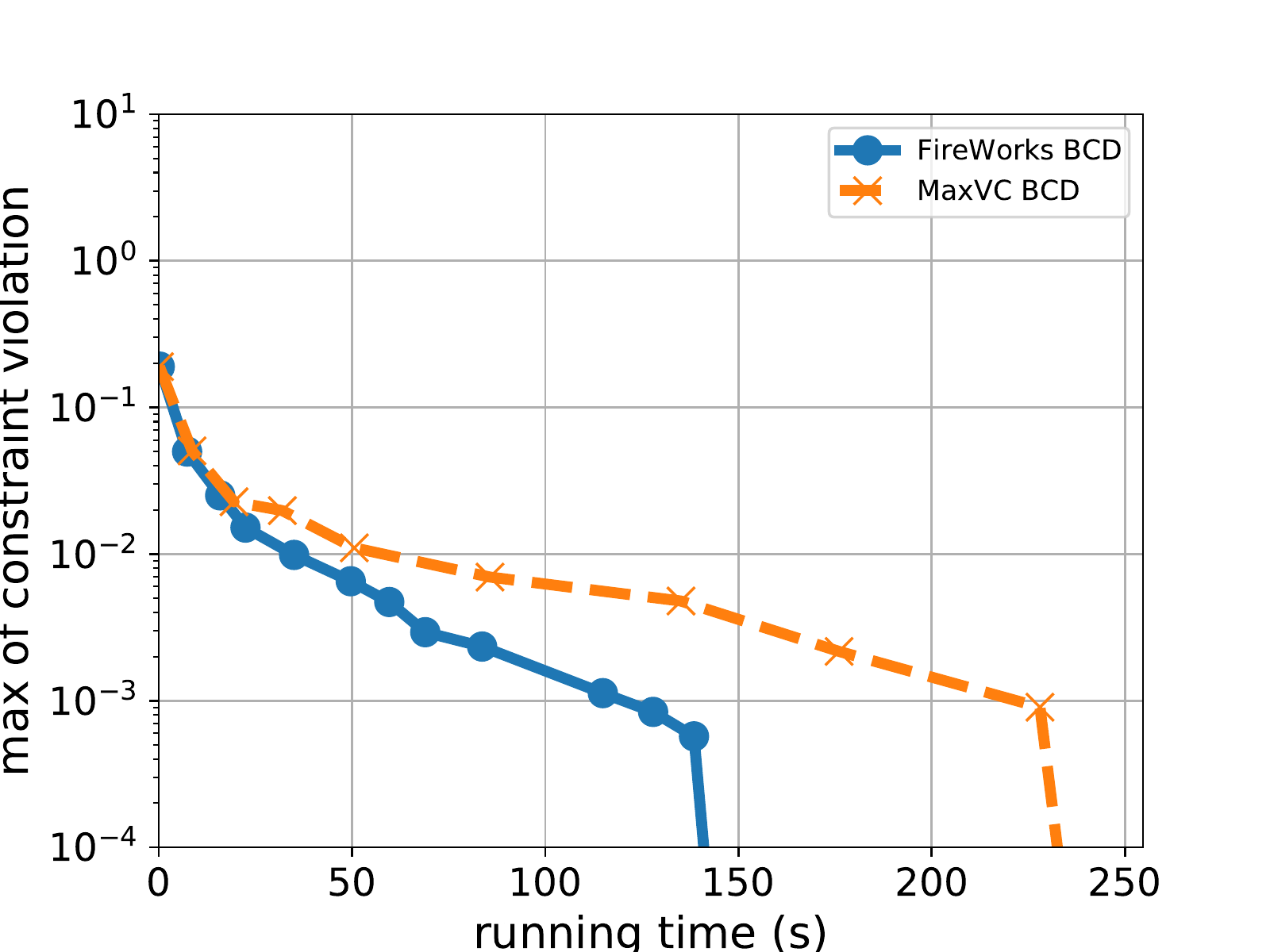}
	\includegraphics[width=3.9cm]{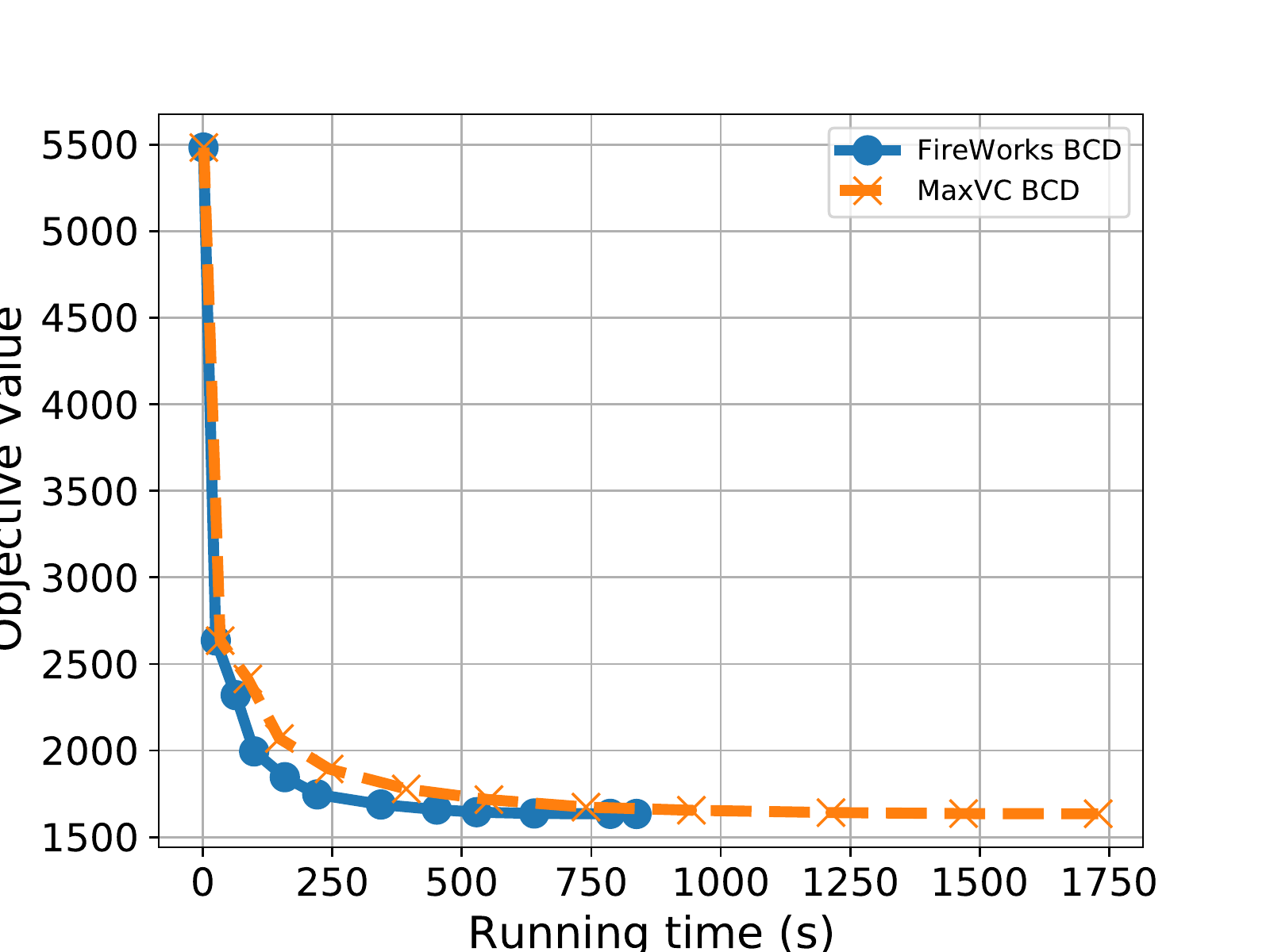}
	\includegraphics[width=3.9cm]{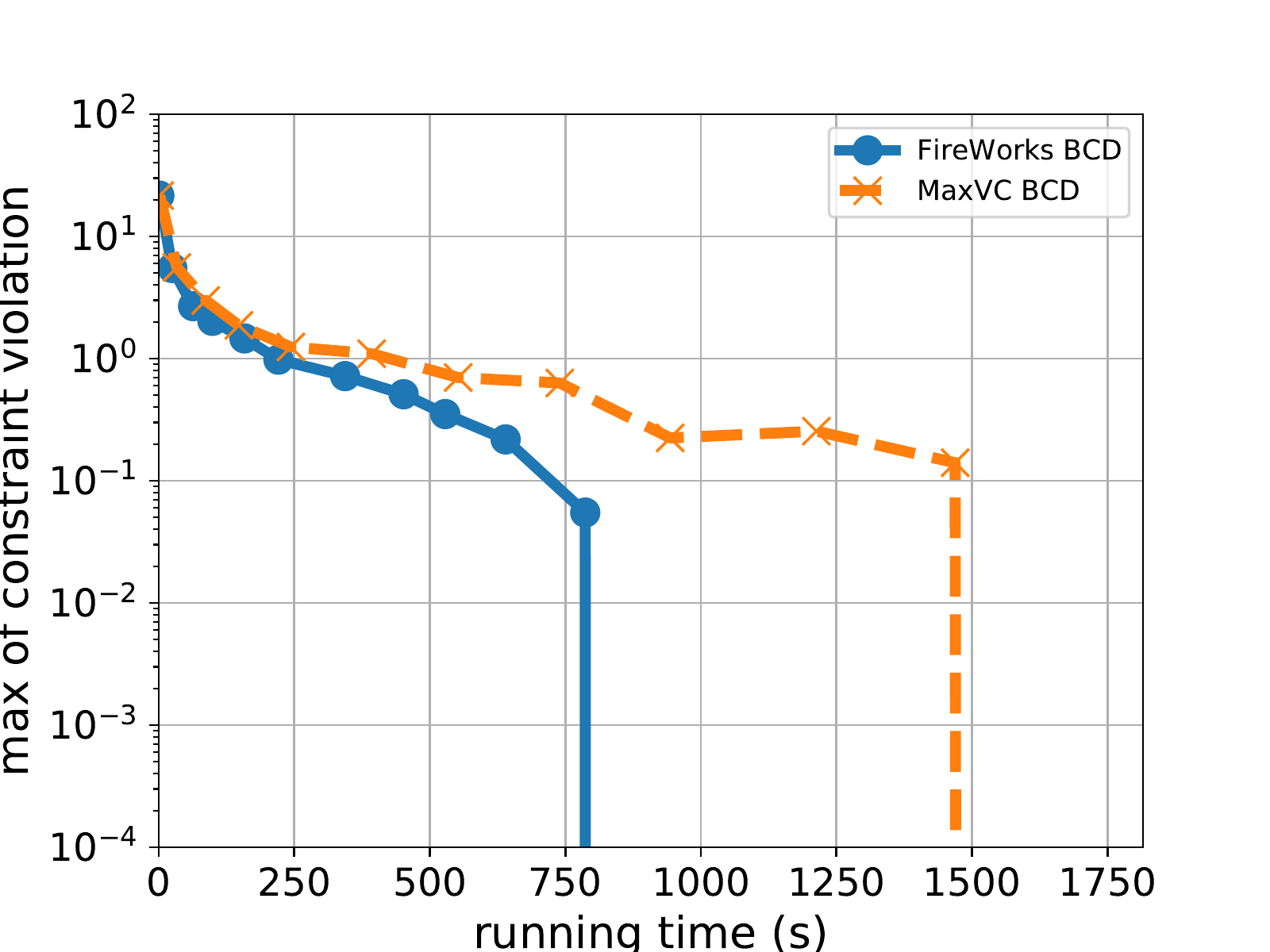}
	\caption{Example of evolution of the objective value and the maximum violation constraint on the $0$-valued weights. The tolerance on the inner problem is set to $10^{-6}$; (most-left) performance on \emph{Newsgroup-3}; (most-right) performance on \emph{Newsgroup-5}.
		\label{fig:objtime}  }
\end{figure*}

\paragraph{Real data}
We have reported comparisons on three real datasets. The first one is the \emph{Leukemia} dataset \citep{golub1999molecular} which has a dense regression matrix with $n=72$ and $d=7129$.  We have also considered sparse problem such as  \emph{newsgroups}  dataset in which we have kept only $3$ categories (\emph{religion}, \emph{atheism} and \emph{graphics}) resulting in $n=1441$, $d=26488$  and  5 categories \emph{comp} leading to  $n=4891$, $d=94414$ (see the supplemental for details).  For these  two problems, we have respectively $223173$  and $676247$ non-zeros elements in the related design matrix $\X$.
We have also used a large-scale dataset which is a subset of the Criteo Kaggle dataset
composed of $2$M samples and $1$M features, with about $78$M non-zero elements in $\X$.
 For \emph{Leukemia}, we have $n_{\rm{added}} = 30$  at each iteration, whereas we have added $300$ and $1000$ features respectively for the \emph{newsgroup} and \emph{Criteo} problem.
Figure \ref{fig:objtime} presents an example of how  objective value and
maximum constraint violation (measured as $\max_j (|\x_j^\top\r_k| - r'_\lambda(0))$)
 evolve during the optimization process for the two \emph{Newsgroup} datasets. We see in these examples that both MaxVC and FireWorks algorithms achieve approximately the same objective value whereas
our {FireWorks} approach converges faster.
Quantitative results are reported in the bottom part of Table \ref{table:toy}.
At first, we can note that the convex relaxation approach using MM and Blitz
is always more efficient than the baseline non-convex methods using either BCD
or GIST. Moreover, the table also
shows that  using FireWorks leads to a speedup of at least one order of magnitude  compared to the baseline algorithm and the MM approach.
For large $\lambda$ leading to sparse solutions, MaxVC is the most efficient approach on \emph{Leukemia}, while for large-scale datasets \emph{newsgroup-3} and \emph{newsgroup-5}, FireWorks is substantially faster than all competitors. For \emph{Criteo}, only the BCD working set
algorithms are able to terminate in reasonable time and FireWorks is more efficient than MaxVC. Again the MM+Blitz approach is performing worse than the
two non-convex active set algorithms and fails to converge in a reasonable time
for large datasets.

\paragraph{Running-time on a grid-search.} We report here the sum of running
time (in seconds) of FireWorks and MaxVC for solving the problem with $10$
different values of $\lambda$ varying from $0.6 \lambda_{\max}$ to $0.01
\lambda_{\max}$  with  $\lambda_{\max}=\max_j |\x_j^\top \y|$
on the Leukemia dataset. For a tolerance of $10^{-5}$, we have
for GIST as inner solver, MaxVC runs in 24.3s and Fireworks takes 18.8s while for BCD as inner
solver, we have for MaxVC 91s and for Fireworks 53.5s. Hence, in this context,
the running time of our approach is still better than the most efficient
competitor. We have similar results for the (small) toy problem. 

\paragraph{Additional experiments in supplementary.} Since our main metric for comparing our algorithm
to competitors is its computational running time, as a sanity check, we have also evaluated the quality of the estimate $\tilde \w^\star$. For instance, for the toy
problem we have measured whether our approach is able to recover the support of the true vector $\w^{\rm{true}}$. Our results show
that there is no approach that outperforms the others under other metrics. This 
makes clear that the gain in running time of FireWorks is not at the expense of worse estimate.
We have also reported some studies that analyze the impact of the parameter
$n_{added}$ (and the related $\tau_k$)  and of pruning 
on the running time of our algorithm FireWorks and on MaxVC. According to our results, the $1\%$ rule seems to be a good heuristic for both algorithms and across the range
of parameters, FireWorks is as efficient as MaxVC.

\section{Conclusion}
\label{sec:conclusion}

We have introduced in this paper a working set based meta-algorithm for non-convex regularized regression. By generalizing the concept of primal-dual approach in a non-convex setting, we
were able to derive a novel rule for updating the features optimized by an
iterative incremental algorithm. From a theoretical point of view, we showed
convergence of the algorithm, even when the inner problem is not
solved exactly but up to a certain tolerance. This is in contrast with the
classical maximal violating optimality condition algorithms approach whose convergence requires the exact resolution of each inner problem. Our experimental results show the
computational gain achieved for a given solver when applied directly on the full
variables or within our working set algorithm. The main limitation of our work is
that our provably convergent method is not always as efficient as heuristic
ones.

\paragraph{Broader and potential negative impact}
{We expect this work to benefit research and applications related to large
scale sparse learning problems. Since the work is a methodological work and as
such it is hard to see any foreseeable societal consequences without precise
applications. The computational gain from our algorithm can be interesting fro
practitioners from a computational (and financial) perspective but it can also
be counterbalanced by the potential use on larger dataset that this can also bring.
}

\section*{Acknowledgments}

This work benefited from the support of the project Chaire AI RAIMO, 3IA Côte d'Azur Investments ANR-19-P3IA-0002 of the French National Research Agency (ANR) and was performed using computing resources of CRIANN (Normandy, France),. This research was produced within the framework of Energy4Climate Interdisciplinary Center (E4C) of IP Paris and Ecole des Ponts ParisTech. This research was supported by 3rd Programme d’Investissements d’Avenir ANR-18-EUR-0006-02. This action benefited from the support of the Chair "Challenging Technology for Responsible Energy" led by l’X – Ecole polytechnique and the Fondation de l’Ecole polytechnique, sponsored by TOTAL.

\providecommand{\CH}{{C.-H}}\providecommand{\JB}{{J.-B}}

\bibliographystyle{plainnat}

\clearpage

\appendix

\section{\centering Supplementary material \\
Provably Convergent Working Set Algorithm for Non-Convex Regularized
Regression}

\begin{table*}[h]
		\vskip 0.15in
	\begin{center}
		\begin{footnotesize}
						\begin{tabular}{l | l | l}
				\hline
				Penalty & $r_\lambda(|w|)$ &  $\partial r_\lambda(|w|)$ \\
				\hline
				Log sum    & $\lambda \log(1 + |w|/\theta)$ & $ \displaystyle \left\lbrace
				\begin{array}{lll}
				\left[\frac{-\lambda}{\theta}, \frac{\lambda}{\theta} \right] & \text{if} & w = 0  \\
				\left \{\lambda \frac{\sign(w)}{\theta + |w|} \right \} & \text{if} & w \neq 0
				\end{array} \right.$
				\\
								\hline
				MCP & $\left\lbrace
				\begin{array}{lll}
				\lambda |w| - \frac{w^2}{2 \theta} & \text{if} & |w| \leq \lambda \theta \\
				\theta \lambda^2/2 & \text{if} & |w| > \theta \lambda
				\end{array} \right.$
				& $\left\lbrace \begin{array}{lll}
				\left[-\lambda, \lambda \right] & \text{if} & w = 0 \\
				\{\lambda \sign(w) - \frac{w}{\theta}\} & \text{if} & 0 < |w| \leq \lambda \theta \\
				\{0 \} & \text{if} & |w| > \theta \lambda
				\end{array} \right.$
				\\
				\hline

				SCAD & $\left\lbrace
				\begin{array}{lll}
				\lambda |w|  & \text{if} & |w| \leq \lambda  \\
				\frac{- w^2 + 2 \theta \lambda |w| - \lambda^2 }{2 (\theta - 1)} & \text{if} & \lambda < |w| \leq \lambda \theta \\
				\frac{\lambda^2 (1+\theta)}{2} & \text{if} & |w| > \theta \lambda
				\end{array} \right.$
				&  $\left\lbrace \begin{array}{lll}
				\left[-\lambda, \lambda \right] & \text{if} & w = 0 \\
				\{\lambda \sign(w)\} & \text{if} & 0 < |w| \leq \lambda \\
				\left\{\frac{- w +  \theta \lambda \sign(w)  }{\theta - 1}\right\} & \text{if} & 0 < |w| \leq \lambda \theta \\
				\{0 \} & \text{if} & |w| > \theta \lambda
				\end{array} \right.$  \\
				\hline
			\end{tabular}
					\end{footnotesize}
	\end{center}
	\caption{Common non-convex penalties  with their sub-differentials. Here $\lambda > 0$, $\theta > 0$ ($\theta > 1$ for MCP, $\theta > 2$ for SCAD).
	}
	\label{tab:ncvx_pen}
	\vskip -0.1in
\end{table*}

\subsection{Maximum-Violating Optimality Condition Working Set Algorithm}
The maximum-violating constraint algorithm is a simple algorithm that solves at 
each iteration a sub-problem with a subset of variables and then add some others that violate the most the statement $\y - \X \tilde \w_{\cA_{k}}^\star \in \cC_{\bar \cA_{k}}$,
where  ``the most" is evaluated in term of distance to each set $\cC_j$,
with $j \in \bar \cA_{k}$. Hence, at
each iteration, we compute all these distances, sort them in descending order and
add to the current working set, the $n_k$ variables that yield to the largest distances. The algorithm is presented below.

\begin{algorithm}[h]
	\caption{Maximum Violating Constraints Algorithm  }
	\label{alg:maxvc_ncvx}
	\begin{algorithmic}[1]
		\REQUIRE{ $\{\X, \y\}$, $\cA_1$ active set, $n_k$ number of
		variables to add at iteration $k$, initial vector $ \tilde \w_{\cA_{0}}$ }
		\ENSURE{$\tilde \w_{\cA_{k}}$}
		\FOR{$k=1,2,\dots $}
		\STATE $\w_{\cA_{k}} = \displaystyle\argmin_{w \in \cA_k} \frac{1}{2}\|\y - \X_{\cA_k} \w \|_2^2 +  \sum_{j \in \cA_k} r_\lambda(|w_j|)$ \hfill \emph{warm-start solver with $\w_{\cA_{k-1}}$}
		\STATE $\r_k = \y - \X_{\cA_k} \w_{\cA_{k}}$  \hfill \emph{current residual}

		\STATE  $\v$ = argsort $ \text{dist}(\r_{k}, \cC_{j}$) in descending order
		\STATE $\cA_{k+1} = \v[1:n_k]  \cup \cA_k$ \hfill \emph{update working set by adding the $n_k$ most violating variables}
		\ENDFOR
		\STATE Build $\tilde \w_{\cA_{k}}$
	\end{algorithmic}
\end{algorithm}

\subsection{Proof of Proposition \ref{prop:optimality} }
\setcounter{proposition}{0}
\begin{proposition} If $\w_\cA^\star$ satisfies Fermat's condition of Problem \eqref{eq:restrictedprob}, then for all $j \in \cA$, we have
	\begin{equation}\label{eq:proplowerbound}
	|\x_j^\top (\y - \X_{\cA} \w_\cA^\star)| \leq  r_\lambda^\prime(0)
	\end{equation}
	where $r_\lambda^\prime$ is the derivative of $r_\lambda$.
\end{proposition}
\begin{proof} At first, note that since the function $r_\lambda$ is
{monotone and}
concave, then its derivative is positive and non-increasing. Hence
	$\forall w \geq 0,\,\,r^\prime_\lambda(w) \leq r^\prime_\lambda(0)$. Now,
	for $j \in \{i \in \cA : w_{i,\cA}^\star = 0\}$, the inequality in Equation 	\ref{eq:proplowerbound}
	 naturally comes from Fermat's condition in Equation \ref{eq:fermat}. When $j \in \{i \in \cA : w_{i,\cA}^\star \neq 0\}$,  we have 
	 	 $\x_j^\top \rev{\res(\w_{\cA}^\star)} 	 = r_\lambda^\prime(|w_{j,\cA}^\star|) $. Taking
	the absolute value of this equation and plugging in the inequality of the derivatives concludes the proof.
\end{proof}

\subsection{Proof of Proposition \ref{prop:optalpha}}
\setcounter{proposition}{1}
\begin{proposition} \label{prop:optalpha} Given a working set $\cA_{k}$ and $\w_{\cA_{k}}^\star$ solving the related restricted problem, $\tilde \w_{\cA_{k}}^\star$ is also optimal for the full problem if and only if $\alpha=1$ (which also means $\s_{k+1} = \r_k$).
\end{proposition}
\begin{proof}
Assume that $\w_{\cA_{k}}^\star$ and $\tilde \w_{\cA_{k}}^\star$	
are optimal respectively for the restricted and the full problem.	Let us show that in this case $\alpha_k=1$. 
Since $\tilde \w_{\cA_{k}}^\star$ is optimal for the full problem, we thus have $\forall j \in \bar{\cA}_k,\, |\x_j^\top(\y-\X\tilde\w_{\cA_k}^\star)| \leq r_\lambda^\prime(0)$.
And thus we have the following equivalent statement
$$
\y - \X \tilde\w_{\cA_k}^\star \in \cC \Leftrightarrow
\y - \X_{\cA_{k}} \w_{\cA_k}^\star \in \cC \Leftrightarrow \r_{k} \in \cC
$$ 
and thus $\alpha_{k}=1$.

Now assume that $\alpha_{k}=1$ and let us show that $\tilde \w_{\cA_{k}}^\star$ is optimal for the full problem. Since $\alpha_k=1$,
we have $\s_{k+1}=\r_{k}$ and thus $\r_k \in \cC$.
The latter means that $\forall j \in \bar{\cA}_k,\, |\x_j^\top(\y-\X_{\cA_k}\w_{\cA_k}^\star)| \leq r_\lambda^\prime(0)$
and thus $\forall j \in \bar{\cA}_k,\, |\x_j^\top(\y-\X\tilde\w_{\cA_k}^\star)| \leq r_\lambda^\prime(0)$.
Given this last property and the definition of $\tilde \w_{\cA_{k}^\star}$ based on $\w_{\cA_{k}}^\star$, we can conclude
that $\tilde \w_{\cA_{k}}^\star$ is optimal for the full problem.
\end{proof}

\subsection{Proof of Lemma \ref{prop:dist}}

The proof follows similar steps as those given by \citet{johnson2015blitz}.

\setcounter{lemma}{0}

\begin{lemma} At step $k \geq 2$, consider a constraint $\cC_{j}$ such that
	$h_j(\r_k) >0$ and $h_j(\s_k) <0$ then
	\begin{align}
	\text{dist}(\r_k, \cC_{j}) \geq \frac{1 - \alpha_k}{\alpha_k} \tau_{k-1} \enspace.
	\end{align}
\end{lemma}
\begin{proof}
Denote as $j$ the index of the function $h_j$ such that $h_j(\r_k)>0$ and $h_j(\s_k) <0$.
Let's $\z_k \in \{\z \in \R^n: h_j(\z)= 0\}.$ The following equality holds
\begin{align}
 \text{dist}(\r_k, \cC_{j}) &= \|\z_k - \r_{k}\|_2 \nonumber\\
 &= \|\z_k - \frac{1}{\alpha_{k}}( \s_{k+1} -{(1-\alpha_{k})}{} \s_{k})\| \nonumber\\
 &= \left \|\z_k - \frac{1}{\alpha_{k}} \s_{k+1} +\frac{1-\alpha_{k}}{
	\alpha_k} \s_{k}\right \| \nonumber\\
 &= \left \|-\z_k + \frac{1}{\alpha_{k}} \s_{k+1} -\frac{1-\alpha_{k}}{
	\alpha_k} \s_{k}\right \| \nonumber\\
 &= \frac{1-\alpha_{k}}{
 	\alpha_k}\left \|-\frac{\alpha_{k}}{1-
 	\alpha_k}\z_k + \frac{1}{1-\alpha_{k}} \s_{k+1} - \s_{k} \right \| \label{prop:eqa}
\end{align}
Note that because $h_j(\r_k)>0$ and $h_j(\s_k) <0$, $\alpha_k \neq 0$ since $h_j$ is a continuous function. 
By construction,  
we have $h_j(\z_k)=0$  as $\z_k$ is a minimizer of the distance and $h_j(\s_{k+1})=0$ as we have chosen $j$ as the index of the set  
that makes $\s_{k+1} \not \in \cC$. Since $h_j(\cdot) \leq 0$ is a convex set and the coefficients $-\frac{\alpha_k}{1-\alpha_k}$ and $\frac{1}{1-\alpha_{k}}$  do not lead to a convex combination of $\z_k$ and $\s_{k+1}$ and hence, we have 
$h_j(-\frac{\alpha_{k}}{1-
	\alpha_k}\z_k + \frac{1}{1-\alpha_{k}} \s_{k+1}) \geq 0$. On the other hand
by construction, we have $\s_{k} \in \cC_{j}$. Furthermore, we have
$\text{dist}_S(\s_{k},\cC_{j}^=) \geq \tau_{k-1}$. Indeed, since $h_j(\r_{k}) >0$,
we have $j \not \in \cA_{k}$ as by construction $\r_k \in \cC_{\cA_k}$ ($\w_{\cA_{k}}$ has been optimized over $\cA_{k}$). Because $j \not \in \cA_{k}$
means that $\text{dist}_S(\s_{k},\cC_{j}^=) \geq \tau_{k-1}$, by definition of the construction of $\cA_{k}$ in Algorithm \ref{alg:activeset_ncvx}.

Now as $h_j(-\frac{\alpha_{k}}{1-
	\alpha_k}\z_k + \frac{1}{1-\alpha_{k}} \s_{k+1}) \geq 0$ and  $\text{dist}_S(\s_{k},\cC_{j}^=) \geq \tau_{k-1}$, the norm in the above equation \eqref{prop:eqa} is lower bounded by $\tau_k$ and we have
$$
 \text{dist}(\r_k, \cC_{j}) \geq \frac{1-\alpha_{k}}{
 	\alpha_k} \tau_{k-1}.
$$
\end{proof}

\subsection{Proof of Theorem \ref{prop:convergence}}

\setcounter{theorem}{0}
\begin{theorem} Suppose that for each step $k$, the algorithm solving the inner problem ensures	a decrease in the objective value in the form
	\begin{align*}
	f(\tilde \w_{\cA_{k+1}}^\star) - f(\tilde \w_{\cA_k}^\star) \leq - \gamma_k \|\tilde \w_{\cA_{k+1}}^\star - \tilde \w_{\cA_k}^\star\|_2^2 \enspace.
	\end{align*}
	with $\forall k, \,\gamma_k \geq \underline{\gamma} > 0$. For the inner solver, we also impose
	that when solving the problem with set $\cA_{k+1}$, the inner solver is warm-started with $\w_{\cA_{k}}^\star$. Assume also that
	$\|\X\|_2 > 0$,  $\tau_k \geq \underline{\tau} >0$  and $h_j$ satisfies assumption in Lemma \ref{prop:dist}, then the sequence of $\alpha_k$ produced by Algorithm \ref{alg:activeset_ncvx} converges towards $1$ and $\forall j,\,\,\lim_{k \rightarrow \infty} |\x_j^\top \r_{k}| \leq r^\prime_\lambda(0)$.
\end{theorem}

\begin{proof} Before going into details, note that pruning  $\w_{\cA_{k}}^\star$ before warm-starting does not affect $f(\tilde \w_{\cA_k}^\star)$, and thus the proof still holds for that situation.	Using results in Proposition \ref{prop:w} and
	Lemma \ref{prop:dist} and the above assumption, we have, for $k \geq 2$,
	\begin{align}
	f(\tilde \w_{\cA_{k+1}}^\star) &\leq f( \tilde \w_{\cA_k}^\star) - \frac{\gamma_k}{\|\X\|_2^2} \left(\frac{1 - \alpha_k}{\alpha_k} \right)^2 \tau_{k-1}^2 \nonumber\\
	&\leq f(\tilde \w_{\cA_2}^\star) - \frac{1}{\|\X\|_2^2}\sum_{\ell=2}^k \gamma_{\ell} \left(\frac{1 - \alpha_{\ell}}{\alpha_{\ell}} \right)^2 \tau_{\ell-1}^2. \nonumber
	\end{align}
	This means that $\frac{1}{\|\X\|_2^2}\sum_{\ell=2}^k \gamma_{\ell} \left(\frac{1 - \alpha_{\ell}}{\alpha_{\ell}} \right)^2 \tau_{\ell-1}^2 \leq f(\tilde \w_{\cA_2}^\star) - f( \tilde \w_{\cA_{k+1}}^\star)$.
	Since $f$ is bounded from below, the right hand side is less than some
	positive constant, hence  $ \sum_{{\ell}=2}^\infty \gamma_j \left(\frac{1 - \alpha_{\ell}}{\alpha_{\ell}} \right)^2 \tau_{\ell-1}^2 < \infty$. Since the latter sum is bounded, it implies that $\gamma_{\ell} \left(\frac{1 - \alpha_{\ell}}{\alpha_{\ell}} \right)^2 \tau_{\ell-1}^2 \rightarrow 0$ as ${\ell} \rightarrow \infty$, and as $\gamma_{\ell} \geq  \underline{\gamma} > 0 $,  $\tau_{\ell} \geq \underline{\tau} > 0$, we have $\lim_{{\ell} \rightarrow \infty} \alpha_{\ell} = 1$.
	Now using the definition of $\s_{k+1}$, we have  $\forall j,\,\x_j^\top \r_k = \frac{1}{\alpha_{k}} \x_j^\top\s_{k+1} - \frac{1 - \alpha_{k}}{\alpha_{k}} \x_j^\top\s_k $. Then, taking the absolute value,  triangle inequality, using the fact that $\forall k,\,\,\s_{k} \in \cC$ and taking the limit concludes the proof.
\end{proof}

\subsection{Experimental analysis}

\subsubsection{Data}

The toy dataset has been built from scratch and can be reproduced from the companion code of the paper.

The Leukemia dataset we have used is available at 
\url{https://web.stanford.edu/~hastie/CASI_files/DATA/leukemia.html}

The Newsgroup dataset is part of the Sklearn dataset package. The 3 categories is composed of the topic :         \textit{talk.religion.misc}, \textit{comp.graphics} and \textit{alt.atheism}. The 5 categories is composed by 
\textit{comp.graphics}, \textit{comp.os.ms-windows.misc}  \textit{comp.sys.ibm.pc.hardware}
\textit{comp.sys.mac.hardware}, \textit{comp.windows.x}.
We have used the natural default train split as proposed by sklearn \cite{scikit-learn} and the features are based on TF-IDF representation (using the tfidf function of sklearn)
keeping  default parameters.

\subsubsection{Comparing on other metrics}

The main contribution of our work is to propose a working set
algorithm for sparse non-convex regression problem with theoretical guarantees of convergence. We have shown that the main benefit of this algorithm is its computational efficiency.
 
We report below some results  on other
metrics.  We want to show  that there is no
approach outperforming the others. For the Large toy problem, we report the objective value  (white background, top) and support recovery F-measure (in percent) (blue
background, middle). For the Leukemia dataset,  once   feature selection has been performed, 
we report the classification accuracy  in percent,  (averaged over 5 trials ) of a linear SVM trained on the non-zero features of a part of the dataset (50/22 sample splits). Remind that for Leukemia, there is a computational gain of more than $30$ between GIST and Fireworks GIST.
\begin{table*}[h]
			\resizebox{\linewidth}{!}{
		\begin{tabular}{l||rrrr||rrrr}
			\hline
			Data - tol - $K$ & MM prox & GIST    & MaxVC Gist  & FireWorks  Gist   & MM BCD &  BCD & MaxVC BCD & FireWorks BCD	 \\ \hline
			
			Toy large - 1.00e-03 - 0.07 & \textbf{75.8$\pm$4.8} & 76.5$\pm$8.4 & 76.5$\pm$8.4 & 76.5$\pm$8.6 & \textbf{75.6$\pm$0.0} & 76.5$\pm$8.5 & 76.5$\pm$8.4 & 76.5$\pm$8.6\\
			Toy large - 1.00e-05 - 0.07 & - & \textbf{76.5$\pm$8.4} & \textbf{76.5$\pm$8.5} & \textbf{76.5$\pm$8.6} & \textbf{75.6$\pm$0.0} & 76.5$\pm$8.4 & 76.5$\pm$8.5 & 76.5$\pm$8.6\\
			
			Toy large - 1.00e-03 - 0.01 & \textbf{11.5$\pm$0.9} & \textbf{11.5$\pm$1.4} & 11.6$\pm$1.4 & \textbf{11.5$\pm$1.4} & \textbf{11.5$\pm$0.0} & \textbf{11.5$\pm$1.4} & 11.6$\pm$1.4 & \textbf{11.5$\pm$1.4}\\
			Toy large - 1.00e-05 - 0.01 & - & \textbf{11.5$\pm$1.4} & \textbf{11.5$\pm$1.4} & \textbf{11.5$\pm$1.4} & \textbf{11.5$\pm$0.0} & \textbf{11.5$\pm$1.4} & \textbf{11.5$\pm$1.4} & \textbf{11.5$\pm$1.4}\\\hline\hline
			\rowcolor{LightCyan}
			Toy large - 1.00e-03 - 0.07 & 43.6$\pm$2.9 & \textbf{44.4$\pm$2.9} & 43.7$\pm$3.9 &  {44.2$\pm$3.5} & 43.1$\pm$0.0 & {44.2$\pm$2.7} & 43.6$\pm$3.4 & \textbf{44.4$\pm$3.6}\\ 			\rowcolor{LightCyan}
			Toy large - 1.00e-05 - 0.07 & - & \textbf{44.4$\pm$2.9} & 42.8$\pm$4.2 & {43.9$\pm$3.2} & 43.6$\pm$0.0 & \textbf{43.9$\pm$2.6} & 42.8$\pm$4.2 & \textbf{43.9$\pm$3.2}\\ 			\rowcolor{LightCyan}
			Toy large - 1.00e-03 - 0.01 & 39.1$\pm$2.3 & 39.1$\pm$1.1 & 38.3$\pm$1.7 & \textbf{39.3$\pm$1.3} & 37.4$\pm$0.0 & 38.4$\pm$1.9 & 38.4$\pm$1.9 & \textbf{39.4$\pm$1.2}\\ 			\rowcolor{LightCyan}
			Toy large - 1.00e-05 - 0.01 & - & 39.4$\pm$1.7 & 39.2$\pm$1.5 & \textbf{39.8$\pm$1.7} & 38.9$\pm$0.0 & 38.7$\pm$1.7 & 39.0$\pm$1.2 & \textbf{39.1$\pm$2.1}\\\hline \hline
			\rowcolor{LightGreen}
			Leukemia - 1.00e-03 - 0.07 & 90.00$\pm$5.3 & \textbf{91.82$\pm$3.4} & 90.00$\pm$5.3 & 90.91$\pm$6.4 & 90.00$\pm$5.3 & 88.18$\pm$6.2 & \textbf{90.91$\pm$6.4} & \textbf{90.91$\pm$6.4}\\ 			 \rowcolor{LightGreen}
			Leukemia - 1.00e-05 - 0.07 & 86.36$\pm$6.4 & \textbf{91.82$\pm$3.4} & 89.09$\pm$4.6 & \textbf{91.82$\pm$5.3} & 87.27$\pm$6.0 & 89.09$\pm$6.8 & \textbf{90.91$\pm$6.4} & 90.00$\pm$7.8\\ 			 \rowcolor{LightGreen}
			
			Leukemia - 1.00e-03 - 0.01 & 95.45$\pm$4.1 & \textbf{96.36$\pm$3.4} & 95.45$\pm$2.9 & 95.45$\pm$4.1 & 95.45$\pm$4.1 & 92.73$\pm$4.6 & 92.73$\pm$4.6 & \textbf{97.27$\pm$2.2}\\ 			 \rowcolor{LightGreen}
			Leukemia - 1.00e-05 - 0.01 & \textbf{96.59$\pm$3.8} & 96.36$\pm$3.4 & 94.55$\pm$3.4 & 93.64$\pm$2.2 & \textbf{95.45$\pm$4.1} & 92.73$\pm$5.5 & 94.55$\pm$3.4 & 93.64$\pm$2.2\\\hline

		\end{tabular}
	}
\end{table*}

\subsubsection{On the effect of the number of features to add}

In working set algorithms, the number of features to add  $n_{added}$ to the working set at each iteration can be considered as an hyperparameter. Usually, one adds one feature at each iteration but it is not clear whether it is an optimal choice. In the results we reported in Table \ref{table:toy}, for the toy problems we fixed  $n_{added}=30$.  We report in Figure \ref{fig:feature} the running time (averaged over $5$ runs)  we obtain for the Large toy problem (which has $5000$ features and $500$ informative ones), 
 with respects to  that parameter $n_{added}$.  Note that we have reported the performance of MaxVC, a version of MaxVC with pruning (feature with zero weights are removed from $\mathcal{A}_k)$ and our FireWorks using a BCD algorithm as an inner solver. .

We remark that for most configurations, adding $1$ feature at a time
is not optimal and a better choice is to add between 20 to 40 features at a time. When comparing the performance of the different algorithms, as we anticipated, FireWorks is mostly as efficient as MaxVC and its variants.
However, we want to emphasize again that MaxVC and its variants are  algorithms without convergence proofs, and thus we believe
that FireWorks achieves the best compromise between theoretical supported and practical efficiency.

\begin{figure*}
		\begin{center}
	\includegraphics[width=8cm]{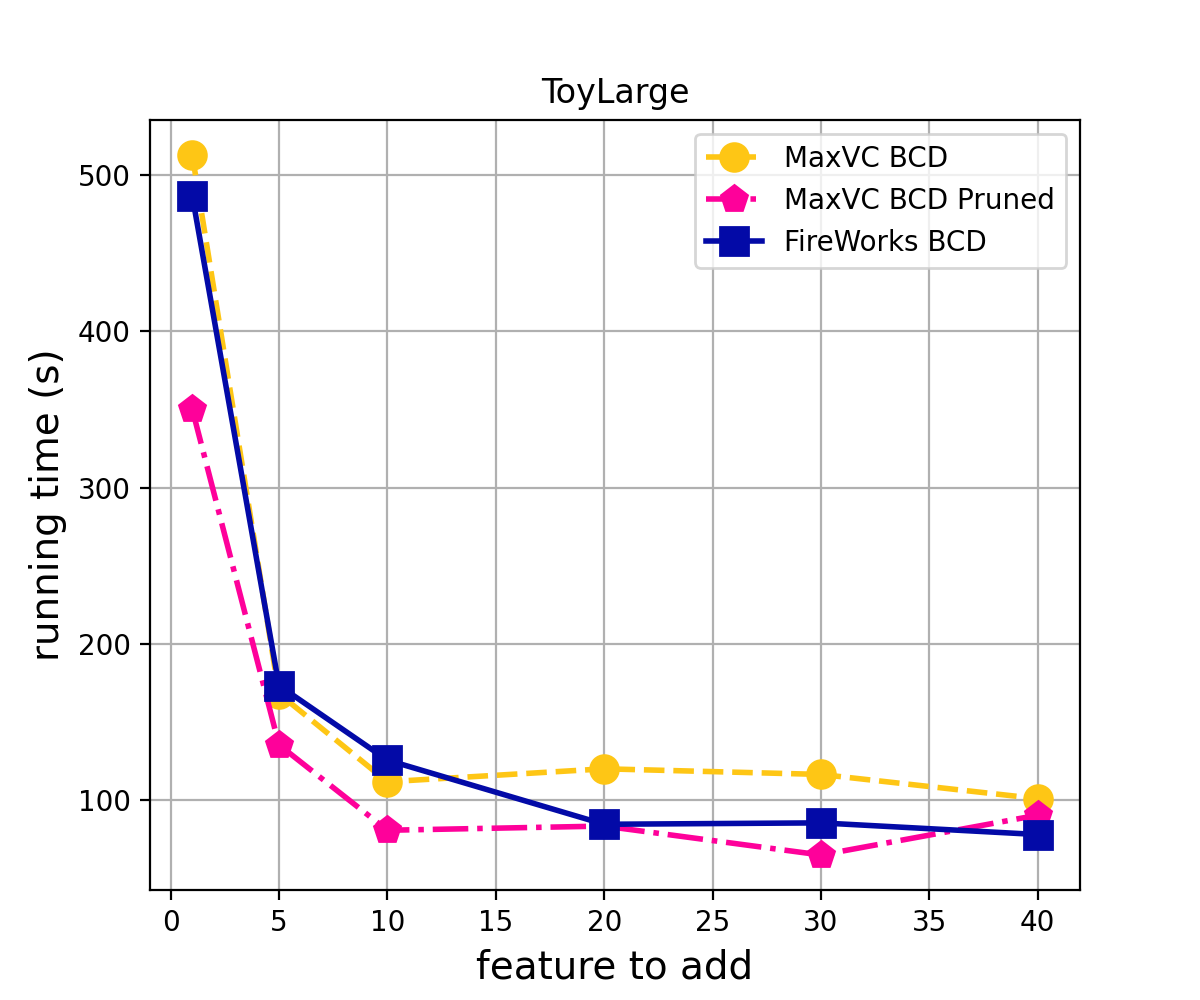}
	\includegraphics[width=8cm]{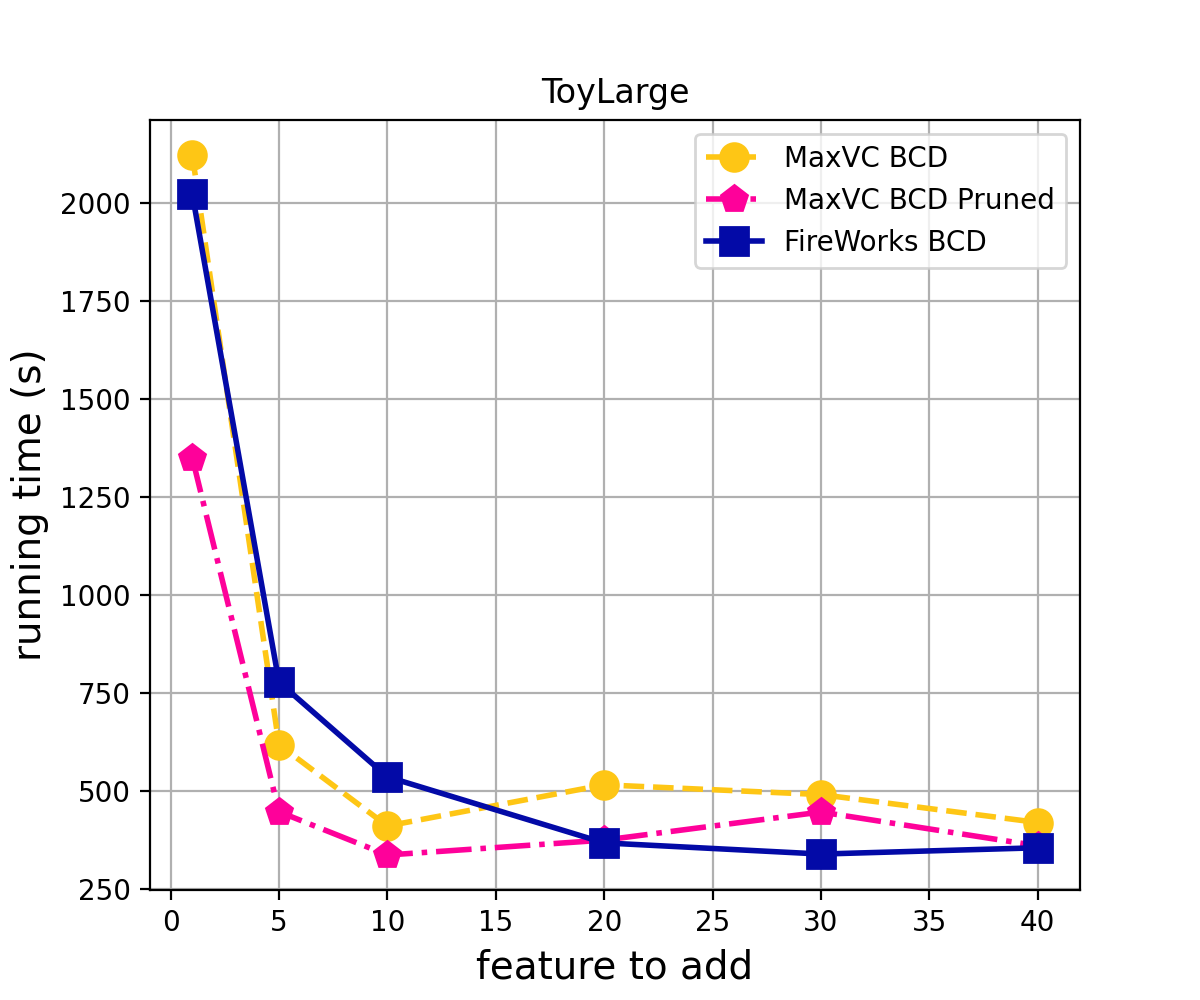} \\
	\includegraphics[width=8cm]{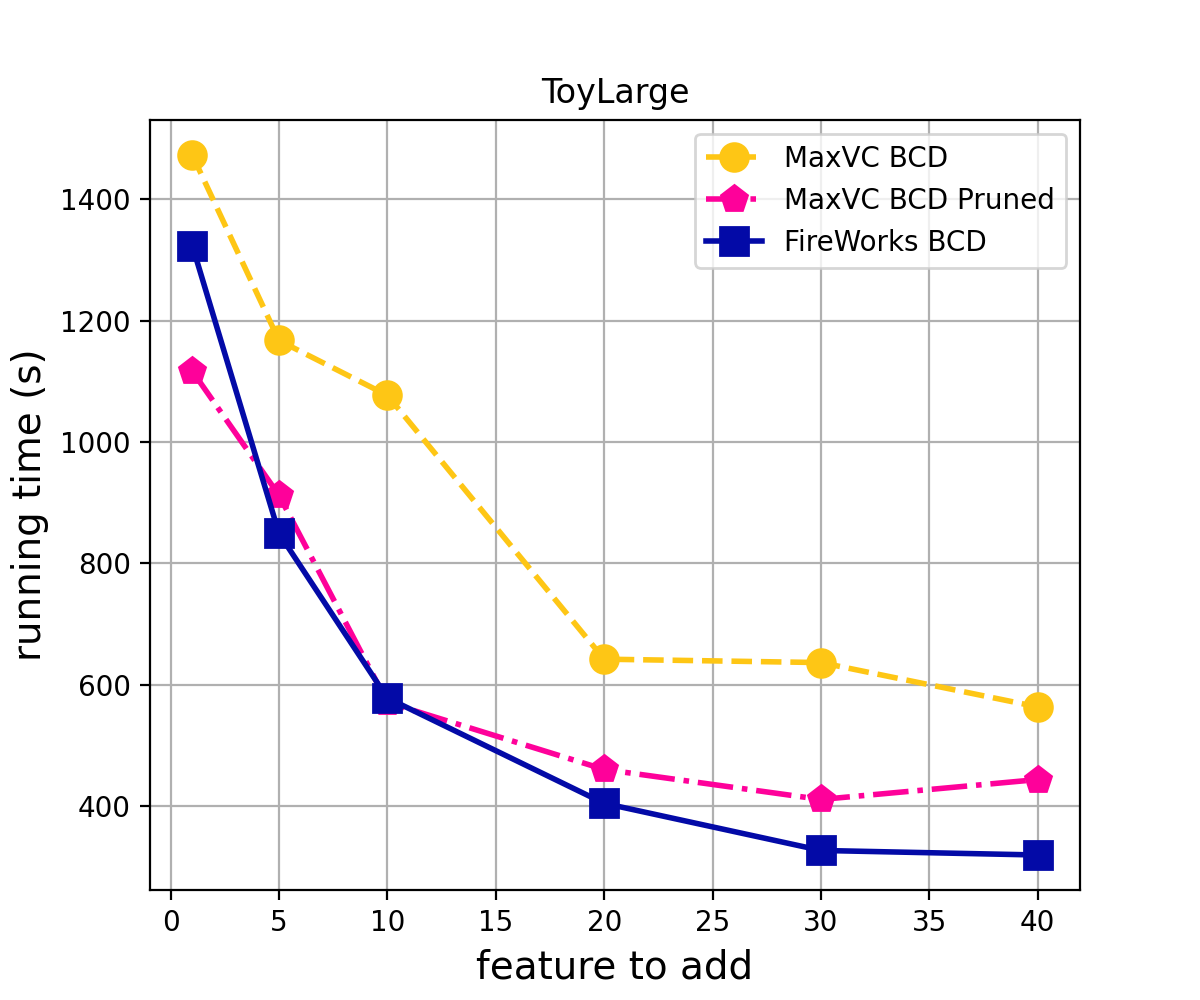}
	\includegraphics[width=8cm]{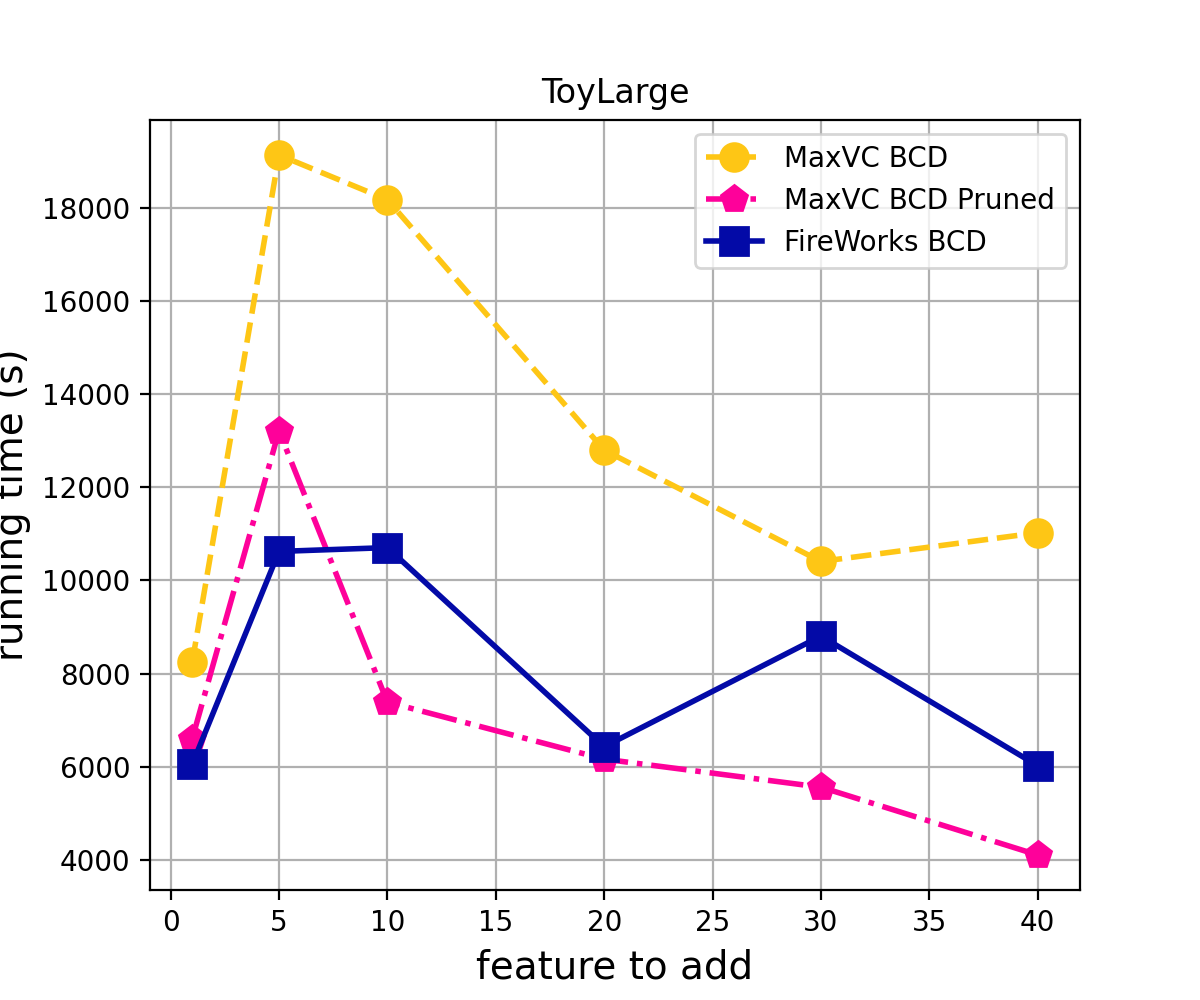}
		\end{center}
\caption{Running time of MaxVC, Max VC with pruning, and our FireWorks on the Large toy problem. The four panels
	varies in the choice of $K$ in the regularization parameter expressed as $\lambda = K \max_j |\x_j^\top \y|$ and in the tolerance $t$ on the stopping criterion .
	 (top-left)   $K=0.07$ and $t=1e^{-3}$ (top-right)~$K=0.07$ and $t=1e^{-5}$.
	 (bottom-left)   $K=0.01$ and $t=1e^{-3}$ (bottom-right)~$K=0.01$ and $t=1e^{-5}$.\label{fig:feature}}.
\end{figure*}

\end{document}